\documentclass[10pt, conference]{IEEEtran}
\usepackage[para,online,flushleft]{threeparttable}
\usepackage{hyperref}
\usepackage{graphicx}
\usepackage{balance}  
\usepackage{booktabs}
\usepackage{etoolbox}
\usepackage{tabularx}
\makeatletter
\patchcmd{\maketitle}{\@copyrightspace}{}{}{}
\makeatother

\let\labelindent\relax
\usepackage{amsmath}
\usepackage{amssymb}
\usepackage{leftidx}
\usepackage{bm}
\usepackage{comment}
\usepackage{color}
\usepackage{cite}
\usepackage{enumitem}
\usepackage{amsthm}

\usepackage{algorithm}
\usepackage{algpseudocode}
\usepackage{multicol}
\usepackage{graphicx}
\usepackage{capt-of}
\usepackage{ragged2e}
\usepackage{etoolbox}
\makeatletter
\patchcmd{\@makecaption}
  {\scshape}
  {}
  {}
  {}
\makeatother

\usepackage{subfigure}
\usepackage{multirow}
\usepackage{cite}
\usepackage{colortbl}
\newtheorem{theorem}{Theorem}[section]
\newtheorem{lemma}[theorem]{Lemma}

\theoremstyle{definition}
\newtheorem{definition}{Definition}[section]
\newtheorem{example}{Example}[section]
\theoremstyle{remark}

\newtheoremstyle{problemstyle}  
        {3pt}                                               
        {3pt}                                               
        {\normalfont\itshape}                               
        {}                                                  
        {\bfseries}                 
        {\normalfont\bfseries:}         
        {.5em}                                          
        {}                                                  
\theoremstyle{problemstyle}

\makeatletter
\newcommand*{\transpose}{%
  {\mathpalette\@transpose{}}%
}
\newcommand*{\@transpose}[2]{%
  \raisebox{\depth}{$\m@th#1\intercal$}%
}
\makeatother

\ifCLASSINFOpdf
\else
\fi
\begin{document}
%
\title{Technical Report: Graph-Structured Sparse Optimization for Connected Subgraph Detection}


\author{\IEEEauthorblockN{Baojian Zhou}
\IEEEauthorblockA{Computer Science Department\\
University at Albany -- SUNY \\
Albany, USA\\
bzhou6@albany.edu}
\and
\IEEEauthorblockN{Feng Chen}
\IEEEauthorblockA{Computer Science Department\\
University at Albany -- SUNY \\
Albany, USA \\
fchen5@albany.edu}
}
\vspace{-5mm}
\maketitle
\begin{abstract}
Structured sparse optimization is an important and challenging problem for analyzing high-dimensional data in a variety of applications such as bioinformatics, medical
imaging, social networks, and astronomy. Although a number of structured sparsity models have been explored, such as trees, groups, clusters, and paths, connected subgraphs have been rarely explored in the current literature. One of the main technical challenges is that there is no \textit{structured sparsity-inducing norm} that can directly model the space of connected subgraphs, and there is no exact implementation of a \textit{projection oracle} for connected subgraphs due to its NP-hardness. In this paper, we explore efficient approximate projection oracles for connected subgraphs, and propose two new efficient algorithms, namely, \textsc{Graph-IHT} and \textsc{Graph-GHTP}, to optimize a generic nonlinear objective function subject to connectivity constraint on the support of the variables. Our proposed algorithms enjoy strong guarantees analogous to several current methods for sparsity-constrained optimization, such as Projected Gradient Descent (\textsc{PGD}), Approximate Model Iterative Hard Thresholding (\textsc{AM-IHT}), and  Gradient Hard Thresholding Pursuit (\textsc{GHTP}) with respect to convergence rate and approximation accuracy. We apply our proposed algorithms to optimize several well-known graph scan statistics in several applications of connected subgraph detection as a case study, and the experimental results demonstrate that our proposed algorithms outperform state-of-the-art methods. 
\end{abstract}


%
\IEEEpeerreviewmaketitle

\section{Introduction}
In recent years, structured sparse methods have attracted much attention in many domains such as bioinformatics, medical imaging, social networks, and astronomy~\cite{bach2012structured, hegde2014fast, jacob2009group, huang2011learning, asterisstay2015icml}. Structured sparse methods have been shown effective to identify latent patterns in high-dimensional data via the integration of prior knowledge about the structure of the patterns of interest, and at the same time remain a mathematically tractable concept. A number of structured sparsity models have been well explored, such as the sparsity models defined through trees~\cite{hegde2014fast}, groups~\cite{jacob2009group}, clusters~\cite{huang2011learning}, and paths~\cite{asterisstay2015icml}. The generic optimization problem based on a structured sparsity model has the form
\setlength{\belowdisplayskip}{4pt} \setlength{\belowdisplayshortskip}{4pt}
\setlength{\abovedisplayskip}{4pt} \setlength{\abovedisplayshortskip}{4pt}
\begin{eqnarray}
\min_{ {\bf x} \in \mathbb{R}^n} f({\bf x})\ \ s.t. \ \ \text{supp}({\bf x}) \in \mathbb{M} \label{problem:general}
\end{eqnarray}
where $f: \mathbb{R}^n \rightarrow \mathbb{R}$ is a differentiable cost function,  the sparsity model $\mathbb{M}$ is defined as a family of structured supports: $\mathbb{M} = \{S_1, S_2, \cdots, S_L\}$, where $S_i \subseteq [n]$ satisfies a certain structure property (e.g., trees, groups, clusters), $[n] = \{1, 2, \cdots, n\}$, and the support set $\text{supp}({\bf x})$ refers to the set of indexes of non-zero entries in ${\bf x}$. For example, the popular $k$-sparsity model is defined as $\mathbb{M} = \{S\subseteq [n]\ |\ |S| \le k\}$. 

Existing structured sparse methods fall into two main categories: 1) \textbf{Sparsity-inducing norms based.} The methods in this category explore structured sparsity models (e.g., trees, groups, clusters, and paths) \cite{bach2012structured} that can be encoded as structured sparsity-inducing norms, and reformulate Problem~(\ref{problem:general}) as a convex (or non-convex) optimization problem
\begin{eqnarray}
\min\nolimits_{{\bf x} \in \mathbb{R}^n} f({\bf x}) + \lambda\cdot  \Omega({\bf x}) \label{problem:convex}
\end{eqnarray}
where $\Omega({\bf x})$ is a structured sparsity-inducing norm of $\mathbb{M}$ that is typically non-smooth and non-Euclidean and $\lambda$ is a trade-off parameter. 2) \textbf{Model-projection based.} The methods in this category rely on a projection oracle of $\mathbb{M}$: \begin{eqnarray}\text{P}({\bf b}) = \arg \min\nolimits_{{\bf x} \in \mathbb{R}^n} \|{\bf b} - {\bf x}\|_2^2\ \ s.t.\ \ \text{supp}({\bf x}) \in \mathbb{M},\end{eqnarray} 
and decompose the problem into two sub-problems, including unconstrained minimization of $f({\bf x})$ and the projection problem $\text{P}({\bf b})$. Most of the methods in this category assume that the projection problem $\text{P}({\bf b})$ can be solved \textbf{exactly}, including the forward-backward algorithm~\cite{zhang2009adaptive}, the gradient descent algorithm~\cite{tewari2011greedy},  the gradient hard-thresholding algorithms~\cite{yuan2013gradient,bahmani2013greedy,jain2014iterative}, the projected iterative hard thresholding~\cite{blumensath2013compressed, bahmani2016learning}, and the Newton greedy pursuit algorithm~\cite{yuan2014newton}. However, when an exact solver of $\text{P}({\bf b})$ is unavailable and we have to apply approximate projections, the theoretical guarantees of these methods do not hold any more. We note that there is one recent approach named as \textsc{Graph}-\textsc{Cosamp} that admits inexact projections by assuming ``head'' and ``tail''  oracles for the projections, but is only applicable to compressive sensing or linear regression problems~\cite{hegde2015nearly}.

We consider an underlying graph $\mathbb{G} = (\mathbb{V}, \mathbb{E})$ defined on the coefficients of the unknown vector ${\bf x}$, where $\mathbb{V} = [n]$ and $\mathbb{E} \subseteq \mathbb{V}\times \mathbb{V}$. We focus on the sparsity model of connected subgraphs that is defined as 
\begin{eqnarray}
\mathbb{M}(\mathbb{G}, k) = \{S \subseteq \mathbb{V}\ |\ |S| \le k, S \text{ is connected}\}, \label{sparsity-model}
\end{eqnarray} 
where $k$ refers to the allowed maximum subgraph size. There are a wide array of applications that involve the search of interesting or anomalous connected subgraphs in networks. The connectivity constraint ensures that subgraphs reflect changes due to localized in-network processes. We describe a few applications below. 

\begin{itemize}
\item \textit{Detection in sensor networks,} e.g., detection of traffic bottlenecks in road networks or airway networks~\cite{anbarouglu2015non}; crime hot spots in geographic networks~\cite{modarres2007hotspot}; and pollutions in water distribution networks~\cite{de2010detection}.
\item \textit{Detection in digital signals and images,} e.g., detection of objects in images~\cite{hegde2015nearly}. 
\item \textit{Disease outbreak detection,} e.g., early detection of disease outbreaks from information networks incorporating data from hospital emergency visits, ambulance dispatch calls and pharmacy sales of over-the-counter drugs~\cite{Speakman-14}.
\item \textit{Virus detection in a computer network,} e.g., detection of viruses or worms spreading from host to host in a computer network~\cite{neil2013scan}. 
\item \textit{Detection in genome-scale interaction network,} e.g., detection of significantly mutated subnetworks~\cite{mairal2013supervised}.
\item \textit{Detection in social media networks,} e.g., detection and forecasting of societal events~\cite{DBLP:conf/kdd/ChenN14, chen2015human}. 
\end{itemize}

To the best of our knowledge, there is no existing approach to Problem~(\ref{problem:general}) for $\mathbb{M}(\mathbb{G}, k)$ that is computationally tractable and provides  performance bound. \textbf{First}, there is no known structured sparsity-inducing norm for $\mathbb{M}(\mathbb{G}, k)$. The most relevant norm is fused lasso norm~\cite{xin2014efficient}: $\Omega({\bf x}) = \sum\nolimits_{(i, j) \in \mathbb{E}} |x_i - x_j|$, where $x_i$ is the $i$-th entry in ${\bf x}$. This norm is able to enforce the smoothness between neighboring entries in ${\bf x}$, but has limited capability to recover all the possible connected subsets as described by $\mathbb{M}(\mathbb{G}, k)$ (See further discussions in Section~\ref{sect:experiments} Experiments). \textbf{Second}, there is no exact solver for the projection oracle of $\mathbb{M}(\mathbb{G}, k)$:\begin{eqnarray}
\text{P}({\bf x}) = \arg \min_{{\bf x} \in \mathbb{R}^n}\|{\bf b} - {\bf x}\|_2^2\ \ s.t.\ \ \text{supp}({\bf x}) \in \mathbb{M}(\mathbb{G}, k),
\label{eqn:projection}
\end{eqnarray} 
as this projection problem is NP-hard due to a reduction from classical Steiner tree problem~\cite{johnson2000prize}. As most existing model-projection based methods require an exact solution to the projection oracle $\text{P}({\bf x})$, these methods are inapplicable to the problem studied here. To the best of our knowledge, there is only one recent approach named as \textsc{Graph}-\textsc{Cosamp} that admits inexact projections for $\mathbb{M}(\mathbb{G}, k)$ by assuming ``head'' and ``tail''  oracles for the projections, but is only applicable to compression sensing and linear regression problems~\cite{hegde2015nearly}. The main contributions of our study are summarized as follows:
\begin{itemize}
\item \textbf{Design of efficient approximation algorithms.} Two new algorithms, namely, \textsc{Graph-IHT} and \textsc{Graph-GHTP}, are developed to approximately solve Problem~(\ref{problem:general}) that has a differentiable cost function and a sparsity model of connected subgraphs $\mathbb{M}(\mathbb{G}, k)$. 
\textsc{Graph-GHTP} is required to minimize $f(\bf x)$ over a projected subspace as an intermediate step, which could be too costly in some applications and \textsc{Graph-IHT} could be considered as a fast variant of \textsc{Graph-GHTP}. 
\item \textbf{Theoretical guarantees and connections.} The convergence rate and accuracy of our proposed algorithms are analyzed under a smoothness condition of $f(\bf x)$ that is more general than popular conditions such as Restricted Strong
Convexity/Smoothness (\textsc{RSC}/\textsc{RSS}) and Stable Mode Restricted Hessian (\textsc{SMRH}). We prove that under mild conditions our proposed  \textsc{Graph-IHT} and \textsc{Graph-GHTP} enjoy rigorous theoretical guarantees. 
\item \textbf{Compressive experiments to validate the effectiveness and efficiency of the proposed techniques.} Both \textsc{Graph-IHT} and \textsc{Graph-GHTP} are applied to optimize a variety of graph scan statistic models for the connected subgraph detection task. Extensive experiments on a number of benchmark datasets demonstrate that \textsc{Graph-IHT} and \textsc{Graph-GHTP} perform superior to state-of-the-art methods that are designed specifically for this task in terms of subgraph quality and running time. 
\end{itemize}

\textbf{Reproducibility:} The implementation of our algorithms and the data sets is open-sourced via the link \cite{fullversion}.

The remaining parts of this paper are organized as follows. Sections~\ref{sect:problem-formulation} introduces the sparsity model of connected subgraphs and statement of the problem. Sections~\ref{sect:algorithms} presents two efficient algorithms and their theoretical analysis. Section ~\ref{sect:applications} discusses applications of our proposed algorithms to graph scan statistic models. Experiments on several real world benchmark datasets are presented in Section~\ref{sect:experiments}. Section~\ref{sect:relatedWork} discusses related work and Section~\ref{sect:conclusion} concludes the paper and describes future work.

\section{Problem Formulation}
\label{sect:problem-formulation}
Given an underlying graph $\mathbb{G} = (\mathbb{V}, \mathbb{E})$ defined on the coefficients of the unknown vector $\bf x$, where $\mathbb{V} = [n]$, $\mathbb{E} \subseteq \mathbb{V}\times \mathbb{V}$, and $n$  is typically large (e.g., $n > 10,000$). The sparsity model of connected subgraphs in $\mathbb{G}$ is defined in~(\ref{sparsity-model}), and its projection oracle $\text{P}({\bf x})$ is defined in~(\ref{eqn:projection}). As this projection oracle is NP-hard to solve, we first introduce efficient approximation algorithms for $\text{P}({\bf x})$ and then present statement of the problem that will be studied in the paper.

\subsection{Approximation algorithms for the projection oracle \texorpdfstring{$\text{P}({\bf x})$}{Px}}
\label{sect:modelOfCS}

There are two nearly-linear time approximation algorithms~\cite{hegde2015nearly} for $\text{P}({\bf x})$ that have the following properties: 
\begin{itemize}
\item \textbf{Tail approximation} ($\text{T}({\bf x})$):  Find a $S\subseteq \mathbb{V}$ such that 
\begin{eqnarray}
\|{\bf x} -{\bf x}_S\|_2 \le c_T \cdot \min_{S^\prime \in \mathbb{M}(\mathbb{G}, k_T) } \|{\bf x} - {\bf x}_{S^\prime}\|_2,
\end{eqnarray}
where $c_T = \sqrt{7}$, $k_T=5k$, and  ${\bf x}_{S}$ is the restriction of ${\bf x}$ to indices in $S$: we have $({\bf x}_{S})_i = x_i$ for $i \in S$ and $({\bf x}_{S})_i = 0$ otherwise. 
\item \textbf{Head approximation} ($\text{H}({\bf x})$): Find a $S \subseteq \mathbb{V}$ such that 
\begin{eqnarray}
\|{\bf x}_S\|_2 \ge c_H\cdot \max_{S^\prime \in \mathbb{M}(\mathbb{G}, k_H) } \|{\bf x}_{S^\prime}\|_2,
\end{eqnarray}
where $c_H = \sqrt{1/14}$ and $k_H = 2k$. 
\end{itemize}
It can be readily proved that, if $c_T = c_H = 1$, then $\text{T}({\bf x}) = \text{H}({\bf x}) = \text{P}({\bf x})$, which indicates that these two approximations ($\text{T}({\bf x})$ and $\text{H}({\bf x})$) stem from the fact that $c_T > 1$ and $c_H < 1$.

\begin{figure*}[t]
  \centering
  \includegraphics[width=1\textwidth,natwidth=610,natheight=642]{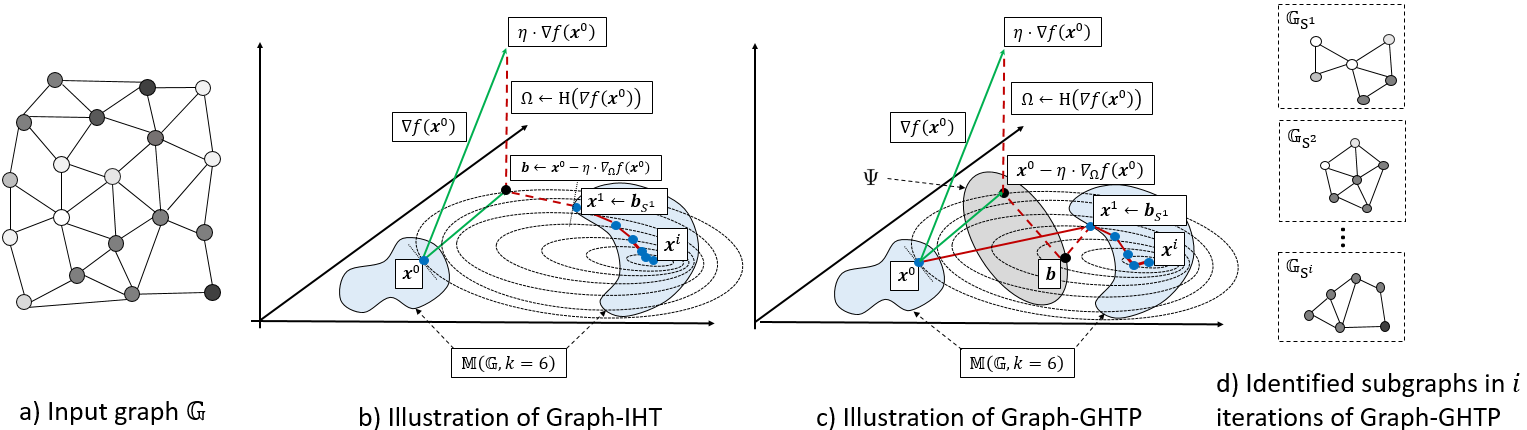}
  \caption{ Illustration of \textsc{Graph-IHT} and \textsc{Graph-GHTP} on the main steps of each iteration. In this example, the gray scale of each node $i$ encodes the weight of this node $w_i \in \mathbb{R}$, $f({\bf x}) = - {\bf w}^\mathsf{T}{\bf x} +  \frac{1}{2} \| {\bf x}\|^2$, and the maximum size of subgraphs is set to $k=6$, where ${\bf w} = [w_1, \cdots, w_n]^T$. The resulting problem tends to find a connected subgraph with the largest overall weight (See discussion about the effect of  $\frac{1}{2}\|x\|^2$ in Section~\ref{sect:applications}). In each iteration, $S^i$ is the connected subset of nodes and its induced subgraph is denoted as $\mathbb{G}_{S^i}$. The sequence of intermediate vectors and subgraphs are: $({\bf x}^0, \mathbb{G}_{S^1}), \cdots, ({\bf x}^i, \mathbb{G}_{S^i})$. } 
  \label{fig:model-example}
  \vspace{-3mm}
\end{figure*}
\subsection{Problem statement}
\label{sect:problemStatement}
Given a predefined cost function $f({\bf x})$ that is differentiable, the input graph $\mathbb{G}$, and the sparsity model of connected subgraph $\mathbb{M}(\mathbb{G}, k)$, the problem to be studied is formulated as: 
\begin{eqnarray}
\min_{{\bf x} \in \mathbb{R}^n} f({\bf x})\ \ s.t. \ \ \text{supp}({\bf x}) \in \mathbb{M}(\mathbb{G}, k)\label{prob:focusedproblem}. 
\end{eqnarray}
Problem~(\ref{prob:focusedproblem}) is difficult to solve as it involves decision variables from a nonconvex set that is composed of many disjoint subsets. In this paper, we will develop nearly-linear time algorithms to approximately solve Problem~(\ref{prob:focusedproblem}). The key idea is decompose this problem to sub-problems that are easier to solve. These sub-problems include an optimization sub-problem of $f({\bf x})$ that is independent on $\mathbb{M}(\mathbb{G}, k)$ and projection approximations for $\mathbb{M}(\mathbb{G}, k)$, including $\text{T}({\bf x})$ and $\text{H}({\bf x})$. We will design efficient algorithms to couple these sub-problems to obtain global solutions to Problem~(\ref{prob:focusedproblem}) with  good trade-off on running time and accuracy. 

\section{Algorithms}
\label{sect:algorithms}

This section first presents two efficient algorithms, namely, \textsc{Graph-IHT} and \textsc{Graph-GHTP}, and then analyzes their time complexities and performance bounds.  

\subsection{Algorithm \textsc{Graph-IHT}}
The proposed \textsc{Graph-IHT} algorithm generalizes the traditional algorithm named as projected gradient descent~\cite{blumensath2013compressed,bahmani2016learning} that requires a exact solver of the projection oracle $\text{P}({\bf x})$. 
The high-level summary of \textsc{Graph-IHT} is shown in Algorithm~\ref{Graph-IHT} and illustrated in Figure~\ref{fig:model-example} (b). The procedure generates a sequence of intermediate vectors ${\bf x}^{0}$, ${\bf x}^{1}$, $\cdots$ from an initial approximation ${\bf x}^{0}$. 
At the $i$-th iteration, the \textbf{first step} (Line 5) first calculates the gradient ``$\nabla f({\bf x}^{i})$'', and then identifies a subset of nodes via head approximation that returns a support set with the head value at least a constant factor of the optimal head value: ``$\Omega \leftarrow \text{H}(\nabla f({\bf x}^i))$''. The support set $\Omega$ can be interpreted as the subspace where the nonconvex set ``$\{ {\bf x}\ |\ \text{supp}({\bf x}) \in \mathbb{M}(\mathbb{G}, k)\}$'' is located, and the projected gradient in this subspace is: ``$\nabla_\Omega f({\bf x}^i)$''. The \textbf{second step} (Line 6) calculates the projected gradient descent at the point ${\bf x}^{i}$ with step-size $\eta$: ``${\bf b} \leftarrow {\bf x}^i - \eta \cdot \nabla_\Omega f({\bf x}^i)$''. The \textbf{third step} (Line 7) identifies a subset of nodes via tail approximation that returns a support set with tail value at most a constant times larger than the optimal tail value: ``$S^{i+1} \leftarrow \text{T}({\bf b})$''. The \textbf{last step} (Line 8) calculates the intermediate solution ${\bf x}^{i+1}$: ${\bf x}^{i+1} = {\bf b}_{S^{i+1}}$. The previous two steps can be interpreted as the projection of ${\bf b}$ to the nonconvex set ``$\{ {\bf x}\ |\ \text{supp}({\bf x}) \in \mathbb{M}(\mathbb{G}, k)\}$'' using the tail approximation. 

\begin{algorithm}[h]
\caption{\textsc{Graph-IHT}}
\label{Graph-IHT}
\begin{algorithmic}[1]
\State \textbf{Input}: Input graph $\mathbb{G}$,  maximum subgraph size $k$, and  step size $\eta$ (1 by default).
\State \textbf{Output}: The estimated vector $\hat{\bf x}$ and the corresponding connected subgraph $\mathbb{S}$. 

\State $i \leftarrow 0$, ${\bf x}^i \leftarrow {\bf 0}$; $S^{i} \leftarrow \emptyset$;
\Repeat
\State $\Omega \leftarrow \text{H}(\nabla f({\bf x}^i))$;
\State ${\bf b} \leftarrow {\bf x}^i - \eta \cdot \nabla_\Omega f({\bf x}^i)$;
\State $S^{i+1} \leftarrow \text{T}({\bf b})$;
\State ${\bf x}^{i+1} \leftarrow {\bf b}_{S^{i+1}}$;
\State $i \leftarrow i+1$;
\Until{halting condition holds}
\State {\bf return} $\hat{{\bf x}} = {\bf x}^{i}$ and $\mathbb{S} = \mathbb{G}_{S^{i}}$;
\end{algorithmic}
\end{algorithm}

\begin{algorithm}[h]
\caption{\textsc{Graph-GHTP}}
\label{Graph-GHTP}
\begin{algorithmic}[1]
\State \textbf{Input}: Input graph $\mathbb{G}$,  maximum subgraph size $k$, and  step size $\eta$ (1 by default).
\State \textbf{Output}: The estimated vector $\hat{\bf x}$ and the corresponding connected subgraph $\mathbb{S}$. 
\State $i \leftarrow 0$, ${\bf x}^i \leftarrow {\bf 0}$; $S^{i} \leftarrow \emptyset$;
\Repeat
\State $\Omega \leftarrow \text{H}(\nabla f({\bf x}^i))$
\State $\Psi \leftarrow \text{supp}({\bf x}^i - \eta \cdot \nabla_\Omega f({\bf x}^i))$;
\State ${\bf b} \leftarrow \arg\min_{{\bf x} \in \mathbb{R}^n} f({\bf x})\ \ s.t. \ \ \text{supp}({\bf x}) \subseteq \Psi$;
\State $S^{i+1} \leftarrow \text{T}({\bf b})$;
\State ${\bf x}^{i+1}  \leftarrow {\bf b}_{S^{i+1}}$;
\State $i \leftarrow i + 1$;
\Until{halting condition holds}
\State {\bf return} $\hat{{\bf x}} = {\bf x}^{i}$ and $\mathbb{S} = \mathbb{G}_{S^{i}}$;
\end{algorithmic}
\end{algorithm}

\subsection{Algorithm \textsc{Graph-GHTP}}
The proposed \textsc{Graph-GHTP} algorithm generalizes the traditional algorithm named as Gradient Hard Threshold Pursuit (\textsc{GHTP}) that is designed specifically for the k-sparsity model: $\mathbb{M} = \{S\subseteq [n]\ |\ |S| \le k\}$~\cite{yuan2013gradient}.  The high-level summary of \textsc{Graph-GHTP} is shown in Algorithm~\ref{Graph-GHTP} and  illustrated in Figure~\ref{fig:model-example} (c). The \textbf{first two steps} (Line 5 and Line 6) in each iteration is the same as the first two steps (Line 5 and Line 6) of \textsc{Graph-IHT}, except that we return the support of the projected gradient descent: ``$\Psi \leftarrow \text{supp}({\bf x}^i - \eta \cdot \nabla_\Omega f({\bf x}^i))$'',  in which pursuing the minimization will be most effective. Over the support set $S$, the function $f$ is minimized to produce an intermediate estimate at the \textbf{third} step (Line 7): ``${\bf b} \leftarrow \arg\min_{{\bf x} \in \mathbb{R}^n} f({\bf x})\ \ s.t. \ \ \text{supp}({\bf x}) \subseteq \Omega$''. 
The \textbf{fourth and fifth} steps (Line 8 and Line 9) are the same as the last two steps (Line 7 and Line 8) of \textsc{Graph-IHT} in each iteration. 

\subsection{Relations between \textsc{Graph-IHT} and \textsc{Graph-GHTP}}

These two algorithms are both variants of gradient descent. In overall, \textsc{Graph-GHTP} converges faster than \textsc{Graph-IHT} as it identifies a better intermediate solution in each iteration by minimizing $f(\bf x)$ over a projected subspace $\{{\bf x} \ | \ \text{supp}({\bf x}) \subseteq \Omega \}$. If the cost function $f({\bf x})$ is linear or has some special structure, this intermediate step can be conducted in nearly-linear time. However, when this step is too costly in some applications, \textsc{Graph-IHT} is preferred. 

\subsection{Theoretical Analysis of \textsc{Graph-IHT}}
In order to demonstrate the accuracy of estimates using
Algorithm 1, we require that the cost function $f({\bf x})$ satisfies
the Weak Restricted Strong Convexity (WRSC) condition as
follows:

\begin{definition}[Weak Restricted Strong Convexity  Property (WRSC)]
A function $f(\bf x)$ has the ($\xi$, $\delta$, $\mathbb{M}$)-model-WRSC if $\forall {\bf x}, {\bf y }\in \mathbb{R}^n$  and $\forall  S \in \mathbb{M}$ with $\text{supp}({\bf x}) \cup \text{supp}({\bf y}) \subseteq S $, the following inequality holds for some $\xi > 0$ and $0 < \delta < 1$: 
\begin{eqnarray}
\| {\bf x} - {\bf y} - \xi \nabla_S f( {\bf x}) + \xi \nabla_S f( {\bf y})\|_2 \le \delta \| {\bf x} - {\bf y} \|_2. 
\end{eqnarray}
\end{definition}

The \textsc{WRSC} is weaker than the popular \textit{Restricted Strong Convexity/Smoothness} (RSC/RSS) conditions that are used in theoretical analysis of convex optimization algorithms~\cite{yuan2013gradient}. The RSC condition basically characterizes cost functions that have quadratic bounds on the derivative of the objective function when restricted to model-sparse vectors. The \textsc{RSC/RSS} conditions imply condition \textsc{WRSC}, which indicates that WRSC is no stronger than \textsc{RSC/RSS}~\cite{yuan2013gradient}. In the special case where $f({\bf x}) = \| {\bf y} - A {\bf x}\|_2^2$ and $\xi = 1$, the condition ($\xi$, $\delta$, $\mathbb{M}$)-model-\textsc{WRSC} reduces to the well known Restricted Isometry Property (\textsc{RIP}) condition in compressive sensing.

\begin{theorem}Consider the sparsity model  of connected subgraphs $\mathbb{M}(\mathbb{G}, k)$ for some $k \in \mathbb{N}$ and a cost function $f: \mathbb{R}^n \rightarrow \mathbb{R}$ that satisfies the $\left(\xi, \delta, \mathbb{M}(\mathbb{G}, 5k)\right)$-model-WRSC condition. If $\eta = c_H(1 - \delta) - \delta$ then for any $ {\bf x} \in \mathbb{R}^n$ such that $\text{supp}({\bf x}) \in \mathbb{M}(\mathbb{G}, k)$, with $\eta > 0$ the iterates of Algorithm 2 obey
\begin{small}
\begin{eqnarray}
\|{\bf x}^{i+1}-{\bf x}\|_2 \le \alpha \|{\bf x}^i-{\bf x}\|_2 + \beta \|\nabla_I f({\bf x})\|_2
\label{equation_10}
\end{eqnarray}
\end{small}
where \begin{small}\[\alpha_0 = c_H (1-\delta) - \delta, \beta_0 = \delta (1+c_H), \] \[\alpha = \frac{\sqrt{2}(1 + c_T)}{1 - \delta} \left(\sqrt{1 - \alpha_0^2} + \left( (2 -\frac{\eta}{\xi})\delta + 1 - \frac{\eta}{\xi}\right) \right),\] \[\beta = \frac{1 + c_T}{1 - \delta} \left( (1+2\sqrt{2})\xi + (2-2\sqrt{2})\eta + \frac{\sqrt{2}\beta_0}{\alpha_0} +  \frac{\sqrt{2} \alpha_0 \beta_0}{\sqrt{1 - \alpha_0^2}}\right),\]\end{small}  
and 
\begin{small}$
I = \arg \max_{S \in \mathbb{M}(\mathbb{G}, 8k)} \|\nabla_{S} f(x)\|_2 
$\end{small}\label{theorem-convergence:Graph-IHT}
\end{theorem}

Before we prove this result, we give the following two lemmas~\ref{lemma:twoinequalities} and ~\ref{lemma:r-Complement}.

\begin{lemma}\cite{yuan2013gradient}
Assume that $f$ is a differentiable function. If $f$ satisfies condition $(\xi, \delta, \mathbb{M})$-WRSC, then $\forall {\bf x}, {\bf y} \in \mathbb{R}^n$ with $\text{supp}({\bf x})\cup \text{supp}({\bf y})\subset S \in \mathbb{M}$, the following two inequalities hold
\vspace{-1mm}
\begin{small}
\begin{eqnarray}
\frac{1 - \delta}{\xi} \|{\bf x} - {\bf y}\|_2 \le \|\nabla_S f({\bf x}) - \nabla_S f({\bf y})\|_2 \le \frac{1 + \delta}{\xi} \|{\bf x} - {\bf y}\|_2 \nonumber \\
f({\bf x}) \le f({\bf y}) + \langle \nabla f({\bf y}), {\bf x} - {\bf y} \rangle + \frac{1+\delta}{2\xi} \|{\bf x} - {\bf y}\|_2^2 \nonumber 
 \end{eqnarray}
 \end{small}\label{lemma:twoinequalities}
 \vspace{-2mm}
\end{lemma}

\begin{lemma}
\label{lemma_2}
Let $\alpha_0 = c_H(1 - \delta) - \delta$, $\beta_0 = \xi(1 + c_H)$, ${\bf r}^i = {\bf x}^i - {\bf x}$, and $\Omega = H(\nabla f({\bf x}^i))$. Then
\begin{small}
\begin{eqnarray}
\|{\bf r}^i_{\Gamma^c}\|_2  \le \sqrt{1 - \alpha_0^2} \|{\bf r}^i\|_2 +\left[\frac{\beta_0}{\alpha_0} + \frac{\alpha_0\beta_0}{\sqrt{1-\alpha_0^2}}\right] \|\nabla_I f({\bf x})\|_2\nonumber 
\end{eqnarray}
\end{small}
 
\vspace{-2mm}
\noindent where \begin{small}
$I = \arg \max_{S \in \mathbb{M}(\mathbb{G}, 8k)} \|\nabla_S f({\bf x})\|_2.$\end{small} We assume that $c_H$ and $\delta$ are such that $\alpha_0 > 0$.  \label{lemma:r-Complement}
\end{lemma}
\begin{proof}
Denote $\Phi = \text{supp}({\bf x}) \in \mathbb{M}(\mathbb{G}, {\bf k}), {\bf \Omega} = {\bf H}(\nabla f({\bf x}^i)) \in \mathbb{M}(\mathbb{G},2k)$, ${\bf r}^i = {\bf x}^i - {\bf x}$, and $\Lambda = \text{supp}({\bf r}^i) \in \mathbb{M}(\mathbb{G},6k)$. The component $\|\nabla_\Gamma f({\bf x}^i)\|_2$ can be lower bounded as
\begin{small}
\begin{eqnarray}
\|\nabla_\Gamma f({\bf x}^i)\|_2 &\ge& c_H \|\nabla_{\Phi} f({\bf x}^i) \|_2 \nonumber \\ &\ge& c_H (\| \nabla_\Phi f({\bf x}^i)- \nabla_\Phi f({\bf x}) \|_2  - \|\nabla_\Phi f({\bf x})\|_2 \nonumber\\
&\ge& \frac{c_H  (1 - \delta)}{\xi}   \|{\bf r}^i\|_2 - c_H \|\nabla_I f({\bf x})\|_2, \nonumber
\end{eqnarray} \end{small}
\noindent where the first inequality follows from the definition of head approximation and the last inequality follows from Lemma~\ref{lemma:twoinequalities} of our paper. The component $\|\nabla_\Gamma f({\bf x}^i)\|_2$ can also be upper bounded as
\vspace{-1mm}
\begin{small}
\begin{eqnarray}
\|\nabla_\Gamma f({\bf x}^i)\|_2 &\le&\frac{1}{\xi} \|\xi \nabla_\Gamma f({\bf x}^i)- \xi\nabla_\Gamma f({\bf x})\|_2 + \|\nabla_\Gamma f({\bf x})\|_2 \nonumber \\
&\le&  \frac{1}{\xi}  \|\xi \nabla_\Gamma f({\bf x}^i) -  \xi \nabla_\Gamma f({\bf x}) - {\bf r}^i_\Gamma + {\bf r}^i_\Gamma\|_2 + \nonumber \\  
&\mathrel{\phantom{=}}& \|\nabla_\Gamma f({\bf x})\|_2 \nonumber\\
&\le &  \frac{1}{\xi}  \| \xi \nabla_{\Gamma\cup \Omega} f({\bf x}^i) -  \xi \nabla_{\Gamma\cup \Omega} f({\bf x}) - {\bf r}^i_{\Gamma\cup \Omega}\|_2 + \nonumber \\
&\mathrel{\phantom{=}}& \|{\bf r}^i_\Gamma\|_2 + \|\nabla_\Gamma f({\bf x})\|_2 \nonumber\\
&\le&\frac{\delta}{\xi} \cdot  \|{\bf r}^i\|_2 + \frac{1}{\xi}\|{\bf r}^i_\Gamma\|_2+ \|\nabla_{I} f({\bf x})\|_2, \nonumber
\end{eqnarray}
\end{small}

\vspace{-3mm}
\noindent where the fourth inequality follows from condition $(\xi, \delta, \mathbb{M}(\mathbb{G},8k))$-WRSC and the fact that ${\bf r}^i_{\Gamma\cup \Omega} = {\bf r}^i$. Combining the two bounds and grouping terms, we obtain the inequality:
\begin{eqnarray}
\|{\bf r}^i_\Gamma\|  \ge \alpha_0 \|{\bf r}^i\|_2 - \xi(1+c_H) \|\nabla_I f({\bf x})\|_2. \nonumber 
\end{eqnarray}
We have $\|{\bf r}^i_\Gamma\|  \ge \alpha_0 \|{\bf r}^i\|_2 - \beta_0 \|\nabla_I f({\bf x})\|_2$. After a number of algebraic manipulations, we obtain the inequality
\begin{small}
\begin{eqnarray}
\|{\bf r}^i_{\Gamma^c}\|_2  \le \sqrt{1 - \alpha_0^2} \|{\bf r}^i\|_2 +\left[\frac{\beta_0}{\alpha_0} + \frac{\alpha_0\beta_0}{\sqrt{1-\alpha_0^2}}\right] \|\nabla_I f({\bf x})\|_2\nonumber ,
\end{eqnarray}
\end{small}
which proves the lemma.
\end{proof}

We give the formal proof of~\ref{theorem-convergence:Graph-IHT}.

\begin{proof}
From the traingle inequality, we have
\begin{eqnarray}
\|{\bf r}^{i+1}\|_2 &=& \|{\bf x}^{i+1} - {\bf x}\|_2 \nonumber \\
&=& \|{\bf b}_{\Psi} - {\bf x} \|_2 \nonumber \\
&\le & \|{\bf b} - {\bf x}\|_2  + \|{\bf b} - {\bf b}_{\Psi}\|_2 \nonumber\\
&\le& (1 + c_T) \|{\bf b} - {\bf x}\|_2 \nonumber \\
&=& (1 + c_T) \|{\bf x}^i - \eta \nabla_{\Omega} f({\bf x}^i) - {\bf x}^i\|_2 \nonumber \\
&=& (1 + c_T) \|{\bf r}^i - \eta \nabla_{\Omega} f({\bf x}^i)\|_2 \nonumber,
\end{eqnarray}
\noindent where $\nabla_{\Omega} f({\bf x}^i)$ is the projeceted vector of $f({\bf x}^i)$ in which the entries outoside $\Omega$ are set to zero and the entries in $\Omega$ are unchanged. $\|{\bf r}^i - \eta \nabla_{\Omega} f({\bf x}^i)\|_2$ has the inequalities
\begin{align}
\|{\bf r}^i - \eta \nabla_{\Omega} f({\bf x}^i)\|_2 = \|{\bf r}_{\Omega^c}^i + {\bf r}_{\Omega}^i - \eta \nabla_{\Omega} f({\bf x}^i)\| \nonumber \\
\leq \|{\bf r}_{\Omega^c}^i \|_2 + \| {\bf r}_{\Omega}^i - \eta \nabla_{\Omega} f({\bf x}^i) + \eta \nabla_{\Omega} f({\bf x}) -\eta \nabla_{\Omega} f({\bf x}) \| \nonumber \\
\leq \|{\bf r}_{\Omega^c}^i \|_2 + \| {\bf r}_{\Omega}^i - \eta \nabla_{\Omega} f({\bf x}^i) + \eta \nabla_{\Omega} f({\bf x}) \| + \| \eta \nabla_{\Omega} f({\bf x}) \| \nonumber \\
\leq \|{\bf r}_{\Omega^c}^i \|_2 + \| {\bf r}_{\Omega}^i - \xi \nabla_{\Omega} f({\bf x}^i) + \xi \nabla_{\Omega} f({\bf x}) \| + \nonumber \\ 
(\xi - \eta) \| \nabla_{\Omega} f({\bf x}^i) - \nabla_{\Omega} f({\bf x}) \|_2 \| + \| \eta \nabla_{\Omega} f({\bf x}) \|_2 \nonumber \\
\leq \|{\bf r}_{\Omega^c}^i \|_2 + (1-\eta / \xi + (2 - \eta / \xi ) \delta) \| {\bf r}^i\|_2 + \eta \| \nabla_{I} f({\bf x})\|_2\nonumber
\end{align}
where the last inequality follows from condition $(\xi,\delta,\mathbb{M})$-WRSC and Lemma~\ref{lemma:twoinequalities}. From Lemma~\ref{lemma:r-Complement}, we have
\begin{eqnarray}
\| {\bf r}_{\Gamma^c}^i\|_2 \leq \sqrt{1 - \alpha_0^2} \| {\bf r}^i \|_2 + \Big[\frac{\beta_0}{\alpha_0} + \frac{\alpha_0 \beta_0}{\sqrt{1 - \alpha_0^2}}\Big] \| \nabla_{I} f({\bf x})\|_2 \nonumber
\end{eqnarray}
Combining the above inequalities, we prove the theorem.
\end{proof}

Our proposed $\textsc{Graph-IHT}$ generalizes several existing sparsity-constrained optimization algorithms: 1) \textbf{Projected Gradient Descent (PGD)}~\cite{rockafellar1976monotone}. If we redefine $\textbf{H}(b) = \text{supp}({\bf b})$ and $\textbf{T}(b) = \text{supp}({\bf P}({\bf b}))$, where $\textbf{P}({\bf b})$ is the projection oracle defined in Equation~(\ref{eqn:projection}), then \textsc{Graph-IHT} reduces to the PGD method; 2) \textbf{Approximated Model-IHT(AM-IHT)}~\cite{hegde2014approximation}. If the cost function $f({\bf x})$ is defined as the least square cost function $f({\bf x}) = \| {\bf y - A x}\|_2^2$, then $\nabla f({\bf x})$ has the specific form $-{\bf A^{\mathsf{T}}(y - A)}$ and \textsc{Graph-IHT} reduces to the AM-IHT algorithm, the state-of-the-art variant of IHT for compressive sensing and linear regression problems. In particular, let ${\bf e = y - Ax}$. The component $\| \nabla f({\bf x^i})\|_2 = \| {\bf A^{\mathsf{T}} e}\|_2$ is upper bound by bounded by $\sqrt{1+\delta} \|{\bf e} \|_2$~\cite{hegde2014approximation}, Assume that $\xi = 1$ and $\eta = 1$. Condition $(\xi,\eta,\mathbb{M})$-WRSC then reduces to the RIP condition in compressive sensing. The convergence inequality~(\ref{equation_10}) then reduces to
\begin{equation}
    \| {\bf x}^{i+1} - {\bf x}\|_2 \leq \alpha' \| {\bf x^i - x}\|_2 + \beta'\| {\bf e}\|_2 ,
\end{equation}
where $\alpha' = (1+c_T) \Big[\delta+\sqrt{1-\alpha_0^2}\Big]$ and 
\begin{equation}
\beta' = (1+c_T) \Big[ \frac{(\alpha_0 + \beta_0) \sqrt{1+ \delta}}{\alpha_0} + \frac{\alpha_0 \beta_0 (\sqrt{1+\delta}) }{\sqrt{1-\alpha_0^2}} \Big]. \nonumber
\end{equation}

Surprisingly, the above convergence inequality \textbf{is identical to the convergence inequality of AM-IHT} derived in~\cite{hegde2014approximation} based on the RIP condition, which indicates that \textsc{Graph-IHT} has the same convergence rate and approximation error as AM-IHT, although we did not make any attempt to explore the special properties of the RIP condition. We note that the convergence properties of \textsc{Graph-IHT} hold in fairly general setups beyond compressive sensing and linear regression. As we consider \textsc{Graph-IHT} as a fast variant of \textsc{Graph-GHTP}, due to space limit we ignore the discussions about the convergence condition of \textsc{Graph-IHT}. The theoretical analysis of \textsc{Graph-GHTP} to be discussed in the next subsection can be readily adapted to the theoretical analysis of \textsc{Graph-IHT}.

\subsection{Theoretical Analysis of \textsc{Graph-GHTP}}
\begin{theorem}Consider the sparsity model  of connected subgraphs $\mathbb{M}(\mathbb{G}, k)$ for some $k \in \mathbb{N}$ and a cost function $f: \mathbb{R}^n \rightarrow \mathbb{R}$ that satisfies the $\left(\xi, \delta, \mathbb{M}(\mathbb{G}, 5k)\right)$-model-WRSC condition. If $\eta = c_H(1 - \delta) - \delta$ then for any $ {\bf x} \in \mathbb{R}^n$ such that $\text{supp}({\bf x}) \in \mathbb{M}(\mathbb{G}, k)$, with $\eta > 0$ the iterates of Algorithm 2 obey
\begin{small}
\begin{eqnarray}
\|{\bf x}^{i+1}-{\bf x}\|_2 \le \alpha \|{\bf x}^i-{\bf x}\|_2 + \beta \|\nabla_I f({\bf x})\|_2
\end{eqnarray}
\end{small}
where \begin{small}\[\alpha_0 = c_H (1-\delta) - \delta, \beta_0 = \delta (1+c_H), \] \[\alpha = \frac{\sqrt{2}(1 + c_T)}{1 - \delta} \left(\sqrt{1 - \alpha_0^2} + \left( (2 -\frac{\eta}{\xi})\delta + 1 - \frac{\eta}{\xi}\right) \right),\] \[\beta = \frac{1 + c_T}{1 - \delta} \left( (1+2\sqrt{2})\xi + (2-2\sqrt{2})\eta + \frac{\sqrt{2}\beta_0}{\alpha_0} +  \frac{\sqrt{2} \alpha_0 \beta_0}{\sqrt{1 - \alpha_0^2}}\right),\]\end{small}  
and 
\begin{small}$
I = \arg \max_{S \in \mathbb{M}(\mathbb{G}, 8k)} \|\nabla_{S} f(x)\|_2 
$\end{small}\label{theorem-convergence:Graph-GHTP}
\end{theorem}

\begin{proof}
Denote $\Omega = \text{H}(\nabla f({\bf x}^i))$ and $\Psi = \text{supp}({\bf x}^i - \eta \cdot \nabla_\Omega f({\bf x}^i))$. Let ${\bf r}^{i+1} = {\bf x}^{i+1} - {\bf x}$. $\|{\bf r}^{i+1}\|_2$ is bounded as 
\vspace{-1mm}
\begin{small}
\begin{eqnarray}
\|{\bf r}^{i+1}\|_2 = \|{\bf x}^{i+1} - {\bf x}\|_2 
&\le & \|{\bf x}^{i+1} - {\bf b}\|_2  + \|{\bf x} - {\bf b}\|_2 \nonumber\\
&\le& c_T \|{\bf x} - {\bf b}\|_2  + \|{\bf x} - {\bf b}\|_2 \nonumber\\
&\le& (1 + c_T) \|{\bf x} - {\bf b}\|_2, \label{proof-theorem63-first-inequality}
\end{eqnarray}
\end{small}

\vspace{-4mm}
\noindent where the second inequality follows from the definition of tail approximation. 
The component $\|({\bf x} - {\bf b})_\Psi \|_2^2$ is bounded as 
\begin{small}
\begin{eqnarray}
\|({\bf x} - {\bf b})_{\Psi}\|_2^2 = \langle {\bf b} - {\bf x}, ({\bf b} - {\bf x})_{\Psi} \rangle \nonumber \\ 
= \langle {\bf b} - {\bf x} - \xi \nabla_{\Psi} f({\bf b}) + \xi \nabla_{\Psi} f({\bf x}), ({\bf b} - {\bf x})_{\Psi}\rangle -    \nonumber \\
 \langle\xi \nabla_{\Psi} f({\bf x}), ({\bf b} - {\bf x})_{\Psi} \rangle \nonumber \\
 \le \delta \|{\bf b} - {\bf x}\|_2 \|({\bf b} - {\bf x})_{\Psi}\|_2 + \xi \|\nabla_{\Psi} f({\bf x})\|_2 \|({\bf b} - {\bf x})_{\Psi}\|_2 \nonumber,
\end{eqnarray}
\end{small}

\vspace{-3mm}
\noindent where the second equality follows from the fact that $\nabla_S f({\bf b}) = {\bf 0}$ since ${\bf b}$ is the solution to the problem in the third Step (Line 7) of \textsc{Graph-GHTP}, and the last inequality can be derived from condition $(\xi, \delta, \mathbb{M}(\mathbb{G}, 8k))$-WRSC. 
After simplification, we have \begin{small}\[\|({\bf x} - {\bf b})_{\Psi}\|_2 \le \delta \|{\bf b} - {\bf x}\|_2 + \xi \|\nabla_{\Psi} f({\bf x})\|_2.\]\end{small} 
\vspace{-3mm}
\noindent It follows that 
\begin{small}\begin{eqnarray}
\|{\bf x} - {\bf b}\|_2 \le \|({\bf x} - {\bf b})_{\Psi}\|_2 + \|({\bf x} - {\bf b})_{{\Psi}^c}\|_2 \nonumber \\
\le \delta \|{\bf b} - {\bf x}\|_2 + \xi \|\nabla_{\Psi} f({\bf x})\|_2 + \|({\bf x} - {\bf b})_{{\Psi}^c}\|_2. \nonumber 
\end{eqnarray}\end{small}

\vspace{-3mm}
\noindent After rearrangement we obtain 
\begin{small}
\begin{eqnarray}
\|{\bf b} - {\bf x}\|_2 &\le& \frac{\|({\bf b} - {\bf x})_{{\Psi}^c}\|_2}{1 - \delta} + \frac{\xi \|\nabla_{\Psi} f({\bf x})\|_2}{1 - \delta}, \label{eqn:theorem63-1}
\end{eqnarray}
\end{small}

\vspace{-3mm}
\noindent where this equality follows from the fact that $\text{supp}({\bf b}) \subseteq S$. 
Let $\Phi = \text{supp}({\bf x}) \in \mathbb{M}(\mathbb{G}, k)$. 
\begin{small}
\begin{eqnarray}\| ({\bf x}^i - \eta \nabla_\Omega f({\bf x}^i))_\Phi \|_2 \le \| ({\bf x}^i - \eta \nabla_\Omega f({\bf x}^i))_{\Psi} \|_2, \nonumber\end{eqnarray}\end{small}

\vspace{-4mm}
\noindent as ${\Psi}=\text{supp}({\bf x}^i - \eta \cdot \nabla_\Omega f({\bf x}^i))$. By eliminating the contribution on $\Phi \cap {\Psi}$, we derive
\begin{small}
\begin{eqnarray}
\|({\bf x}^i - \eta  \nabla_\Omega f({\bf x}^i))_{\Phi\setminus {\Psi}}\|_2 \le \|({\bf x}^i - \eta  \nabla_\Omega f({\bf x}^i))_{{\Psi}\setminus \Phi}\|_2 \nonumber 
\end{eqnarray}
\end{small}
\vspace{-4mm}

\noindent For the right-hand side, we have 
\vspace{-2mm}
\begin{small}\begin{eqnarray}
\|({\bf x}^i - \eta  \nabla_{\Omega} f({\bf x}^i))_{{\Psi}\setminus \Phi}\|_2 
\le\nonumber \\
 \|({\bf x}^i  - {\bf x} - \eta \nabla_{\Omega} f({\bf x}^i)+ \eta \nabla_{\Omega} f({\bf x}))_{{\Psi}\setminus \Phi} \|_2  + \eta \|\nabla_{{\Omega} \cup {\Psi}} f({\bf x})\|_2,\nonumber
\end{eqnarray}\end{small}

\vspace{-4mm}
\noindent where the inequality falls from the fact that $\Phi = \text{supp}({\bf x})$. From the left-hand side, we have 
\vspace{-2mm}
\begin{small}\begin{eqnarray}
\|({\bf x}^i - \eta \nabla_\Omega f({\bf x}^i))_{\Phi \setminus {\Psi}}\|_2 \le  - \eta  \|\nabla_{\Omega \cup \Phi} f({\bf x})\|_2 + \nonumber \\
\|({\bf x}^i - {\bf x} - \eta  \nabla_\Omega f({\bf x}^i) + \eta   \nabla_\Omega f({\bf x}))_{\Phi \setminus {\Psi}} + ({\bf x} - {\bf b})_{{\Psi}^c}\|_2  \nonumber
\end{eqnarray}\end{small}

\vspace{-4mm}
\noindent where the inequality follows from the fact that ${\bf b}_{{\Psi}^c} = {\bf 0}$, ${\bf x}_{\Phi \setminus {\Psi}} = {\bf x}_{{\Psi}^c}$, and $-{\bf x}_{\Phi \setminus {\Psi}} + ({\bf x} - {\bf b})_{{\Psi}^c} = {\bf 0}$. Let $\Phi \Delta {\Psi}$ be the symmetric difference of the set $\Phi$ and ${\Psi}$. It follows that
\begin{small}
\begin{eqnarray}
\|({\bf b} - {\bf x})_{{\Psi}^c}\|_2  \nonumber \\
\le\sqrt{2}\|({\bf x}^i  - {\bf x} - \eta  \nabla_\Omega f({\bf x}^i)+ \eta \nabla_\Omega f({\bf x}))_{\Phi \Delta {\Psi}} \|_2 + 2 \eta \|\nabla_{I} f({\bf x})\|_2 \nonumber \\
\le \sqrt{2}\|({\bf x}^i  - {\bf x} - \xi  \nabla_\Omega f({\bf x}^i)+ \xi \nabla_\Omega f({\bf x}))_{\Phi \Delta {\Psi}} \|_2 + \nonumber \\
\sqrt{2}(\xi-\eta)\| (\nabla_\Omega f({\bf x}^i)+ \nabla_\Omega f({\bf x}))_{\Phi \Delta {\Psi}}\| + 2 \eta \|\nabla_{I} f({\bf x})\|_2  \nonumber\\
\le \sqrt{2}\|({\bf r}^i_{\Omega^c} + {\bf r}^i_{\Omega} - \xi  \nabla_\Omega f({\bf x}^i)+ \xi \nabla_\Omega f({\bf x}))_{\Phi \Delta {\Psi}} \|_2 + \nonumber\\ \sqrt{2}(\xi-\eta)\| (\nabla_\Omega f({\bf x}^i) -  \nabla_\Omega f({\bf x}))_{\Phi \Delta {\Psi}}\| + 2 \eta \|\nabla_{I} f({\bf x})\|_2 \nonumber \\
\le \sqrt{2}\|{\bf r}^i_{\Omega^c} \|+ \sqrt{2}\|({\bf r}^i_{\Omega} - \xi \nabla_\Omega f({\bf x}^i)+ \xi \nabla_\Omega f({\bf x}))_{\Psi \Delta {\Psi}} \|_2 +  \nonumber \\
\sqrt{2}(\xi-\eta )\| (\nabla_\Omega f({\bf x}^i)- \nabla_\Omega f({\bf x}))_{\Psi \Delta {\Psi}}\| + 2 \eta \|\nabla_{I} f({\bf x})\|_2 \nonumber \\
\le \sqrt{2}\|{\bf r}^i_{\Omega^c} \|+ \sqrt{2}\|{\bf r}^i - \xi \nabla_{\Omega\cup {\Psi}\cup \Phi} f({\bf x}^i)+ \xi \nabla_{\Omega\cup {\Psi}\cup \Phi} f({\bf x}) \|_2 +  \nonumber \\
\sqrt{2}(\xi-\eta )\| (\nabla_{\Omega\cup {\Psi}\cup \Phi} f({\bf x}^i)- \nabla_{\Omega\cup {\Psi}\cup \Phi} f({\bf x}))_{\Psi \Delta {\Psi}}\| + 2 \eta \|\nabla_{I} f({\bf x})\|_2 \nonumber \\
\le \sqrt{2}\|r^i_{{\Omega}^c}\|_2 +  \sqrt{2} \left(\left(2 -\frac{\eta}{\xi}\right )\delta + 1 - \frac{\eta}{\xi}\right) \|r^i\| + \nonumber \\  2 \left( \sqrt{2}\xi + (1-\sqrt{2})\eta \right)\|\nabla_{I} f(x)\|_2 \nonumber \label{eqn:theorem63-2},
\end{eqnarray}\end{small}
\noindent where the first inequality follows from the fact that \[ \eta \|\nabla_{\Omega \cup \Phi} f({\bf x})\|_2 + \eta \|\nabla_{{\Psi}\cup \Phi\cup \Omega} f({\bf x})\|_2 \le 2 \eta \|\nabla_{I} f({\bf x})\|_2,\] the third inequality follows as ${\bf x}^i - {\bf x} = {\bf r}^i = {\bf r}^i_{\Omega^c} + {\bf r}^i_{\Omega}$, 
the fourth inequality follows from the fact that $\|({\bf r}^i_{\Omega^c})_{\Phi \Delta {\Psi}}\|_2 \le \|{\bf r}^i_{\Omega^c}\|_2$, the fifth inequality follows as ${\bf r}^i \subseteq \Omega\cup {\Psi}\cup \Phi$, and the last inequality follows from condition $(\xi, \delta, \mathbb{M}(\mathbb{G}, 8k))$-WRSC and Lemma~\ref{lemma:twoinequalities}. From Lemma~\ref{lemma:r-Complement}, we have 
\vspace{-2mm}
\begin{small}
\begin{eqnarray}
\|{\bf r}^i_{{\Omega}^c}\|_2 \le \sqrt{1 - \eta^2} \|{\bf r}^i\|_2 + \left[\frac{\xi(1 + c_H)}{\eta} + \frac{\xi \eta (1 + c_H)}{\sqrt{1 - \eta^2}}\right] \|\nabla_I f({\bf x})\|_2 \nonumber \label{eqn:theorem63-3}
\end{eqnarray}
\end{small}
\noindent Combining~(\ref{eqn:theorem63-1}) and above inequalities, we prove the theorem. 
\end{proof}

Theorem~\ref{theorem-convergence:Graph-GHTP} shows the estimator error of \textsc{Graph-GHTP} is determined by the multiple of $\| \nabla_{S} f({\bf x})\|_2$, and the convergence rate is geometric. Specifically, if ${\bf x}$ is an uncontrained minimizer of $f({\bf x})$, then $\nabla f({\bf x}) = 0$. It means \textsc{Graph-GHTP} is guaranteed to obtain the true ${\bf x}$ to arbitrary precision. The estimation error is negligible when ${\bf x}$ is sufficiently close to an unconstrained minimizer of $f({\bf x})$ as $\| \nabla_S f({\bf x})\|_2$ is a small value.
The parameter \[\alpha = \frac{\sqrt{2}(1 + c_T)}{1 - \delta} \left(\sqrt{1 - \alpha_0^2} + \left( (2 -\frac{\eta}{\xi})\delta + 1 - \frac{\eta}{\xi}\right) \right) < 1,\] controls the convergence rate of \textsc{Graph-GHTP}. Our algorithm allows  an exact recovery if $\alpha < 1$. As $\delta$ is an arbitrary constant parameter, it can be an arbitrary small positive value. Let $\eta$ be $\xi$ and $\delta$ be an arbitrary small positive value, the parameters $c_H$ and $c_T$ satisfy the following inequality
\begin{eqnarray}
c^2_H > 1 - 1/ (1+c_T)^2.
\label{condtion_12}
\end{eqnarray}
It is noted that the head and tail approximation algorithms described in~\cite{hegde2015nearly} do not meet the inequality~(\ref{condtion_12}). Nonetheless, the approximation factor $c_H$ of any given head approximation algorithm can be boosted to any arbitrary constant $c_H^\prime < 1$, which leads to the satisfaction of the above condition as shown in~\cite{hegde2015nearly}.  Boosting the head-approximation algorithm, though strongly suggested by~\cite{hegde2014approximation}, is not empirically necessary.

Our proposed \textsc{Graph-GHTP} has strong connections to the recently proposed algorithm named as Gradient Hard Thresholding Pursuit (\textsc{GHTP})~\cite{yuan2013gradient} that is designed specifically for the k-sparsity model: $\mathbb{M} = \{S\subseteq [n]\ |\ |S| \le k\}$. In particular, if we redefine $\text{H}({\bf b}) = \text{supp}({\bf b})$ and $\text{T}({\bf b})= \text{supp}(\text{P}({\bf b}))$, where $\text{P}({\bf b})$ is the projection oracle defined in Equation~(\ref{eqn:projection}), and assume that there is an algorithm that solves the projection oracle exactly, in which the sparsity model does not require to be the $k$-sparsity model. It then follows that the upper bound of $\|{\bf r}^i_{\Omega^c}\|_2$ stated in Lemma~\ref{lemma:twoinequalities} in Appendix is updated as $\|{\bf r}^i_{\Omega^c}\|_2 \le 0$, since $\text{supp}({\bf r}^i) = \Omega$ and ${\bf r}^i_{\Omega^c} = {\bf 0}$. In addition, the multiplier $(1+c_T)$ is replaced as $1$ as the first inequality~(\ref{proof-theorem63-first-inequality}) in the proof of Theorem~\ref{theorem-convergence:Graph-GHTP} in Appendix is updated as $\|{\bf r}^{i+1}\|_2 \le \|{\bf x} - {\bf b}\|_2$, instead of the original version $\|{\bf r}^{i+1}\|_2 \le  (1 + c_T) \|{\bf x} - {\bf b}\|_2$. After these two changes, the shrinkage rate $\alpha$ is updated as 
\begin{eqnarray}\alpha = \frac{\sqrt{2}}{1 - \delta}\left((2 -\frac{\eta}{\xi})\delta + 1 - \frac{\eta}{\xi} \right),\label{shrinkageratespecial}
\end{eqnarray}
which is \textbf{the same as the shrinkage rate of \textsc{Graph-GHTP}} as derived in~\cite{yuan2013gradient} specifically for the k-sparsity model. The above shrinkage rate $\alpha$~(\ref{shrinkageratespecial}) should satisfy the condition $\alpha< 1$ to ensure the geometric convergence of \textsc{Graph-GHTP}, which implies that 
\begin{eqnarray}\eta > ((2\sqrt{2} + 1)\delta + \sqrt{2} - 1)\xi / (\sqrt{2} + \sqrt{2}\delta).\end{eqnarray}It follows that if $\delta < 1/(\sqrt{2} + 1)$, a step-size $\eta < \xi$ can always be found to satisfy the above inequality. This constant condition of $\delta$ is analogous to the constant condition of state-of-the-art compressive sensing methods that consider noisy measurements~\cite{needell2009cosamp} under the assumption of the RIP condition. We derive the analogous constant using the WRSC condition that weaker than the RIP condition. 

As discussed above, our proposed \textsc{Graph-GHTP} has connections to \textsc{GHTP} on the shrinkage rate of geometric convergence. We note that the shrinkage rate of our proposed \textsc{Graph-GHTP} stated in Theorem~\ref{theorem-convergence:Graph-GHTP} is derived based on head and tail approximations of the sparsity model  of connected subgraphs $\mathbb{M}(\mathbb{G}, k)$, instead of the k-sparsity model that has an exact projection oracle solver. Our convergence properties hold in fairly general setups beyond k-sparsity model, as a number of popular structured sparsity models such as the ``standard'' $k$-sparsity, block sparsity, cluster sparsity, and tree sparsity can be encoded as special cases of $\mathbb{M}(\mathbb{G}, k)$.

\begin{theorem}
Let ${\bf x} \in \mathbb{R}^n$ such that $\text{supp}({\bf x}) \in \mathbb{M}(\mathbb{G}, k)$, and $f: \mathbb{R}^n \rightarrow \mathbb{R}$ be cost function that satisfies condition $\left(\xi, \delta, \mathbb{M}(8k, g)\right)$-WRSC. Assuming that $\alpha < 1$, \textsc{Graph-GHTP} (or \textsc{Graph-IHT}) returns a $\hat{\bf x}$ such that,  $\text{supp}(\hat{\bf x}) \in \mathbb{M}(5k, g)$ and $\|{\bf x} - \hat{\bf x}\|_2 \le c \|\nabla_I f({\bf x})\|_2$, where $c= (1+\frac{\beta}{1-\alpha})$ is a fixed constant. Moreover, \textsc{Graph-GHTP} runs in time 
\begin{eqnarray}
O\left((T+|\mathbb{E}|\log^3 n) \log (\|{\bf x}\|_2/ \|\nabla_I f({\bf x})\|_2)\right) \label{timecomplexity},
\end{eqnarray}
where $T$ is the time complexity of one execution of the subproblem in Step 6 in \textsc{Graph-GHTP} (or Step 5 in \textsc{Graph-IHT}). In particular, if $T$ scales linearly with $n$, then \textsc{Graph-GHTP} (or \textsc{Graph-IHT}) scales nearly linearly with $n$. 
\label{Theorem:runningTime}
\end{theorem}
\begin{proof}
The i-th iterate of \textsc{Graph-GHTP} (or \textsc{Graph-IHT}) satisfies
\vspace{-1mm}
\begin{small}
\begin{eqnarray}
\|{\bf x} - {\bf x}^i\|_2 \le \alpha^i \|{\bf x}\|_2 + \frac{\beta}{1-\alpha} \|\nabla_I f({\bf x})\|_2.
\end{eqnarray}
\end{small}

\vspace{-3mm}
\noindent After $t = \left \lceil \log \left(\frac{\|{\bf x}\|_2}{\|\nabla_I f({\bf x})\|_2}\right) / \log \frac{1}{\alpha} \right \rceil $ iterations,
\textsc{Graph-GHTP} (or \textsc{Graph-IHT}) returns an estimate $\hat{x}$ satisfying 
$
\|{\bf x} - \hat{{\bf x}}\|_2 \le (1 + \frac{\beta}{1 - \alpha}) \|\nabla_I f({\bf x})\|_2.  
$ The time complexities of both head approximation and tail approximation are $O(|\mathbb{E}| \log^3 n)$. The time complexity of one iteration in \textsc{Graph-GHTP} (or \textsc{Graph-IHT}) is $(T+|\mathbb{E}|\log^3 n)$, and the total number of iterations is $\left \lceil \log \left(\frac{\|{\bf x}\|_2}{\|\nabla_I f({\bf x})\|_2}\right) / \log \frac{1}{\alpha} \right \rceil $, and the overall time follows. \textsc{Graph-GHTP} and \textsc{Graph-IHT} are only different in the definition of $\alpha$ and $\beta$ in this Theorem. 
\end{proof}

As shown in Theorem~\ref{timecomplexity}, the time complexity of \textsc{Graph-GHTP} is dependent on the total number of iterations and the time cost ($T$) to solve the subproblem in Step 6. In comparison, the time complexity of \textsc{Graph-IHT}  is dependent on the total number of iterations and the time cost $T$ to calculate the gradient $\nabla f({\bf x}^i)$ in Step 5. It implies that, although \textsc{Graph-GHTP} converges faster than \textsc{Graph-IHT}, the time cost to solve the subproblem in Step 6 is often much higher than the time cost to calculate a gradient $\nabla f({\bf x}^i)$, and hence \textsc{Graph-IHT} runs faster than \textsc{Graph-GHTP} in practice. 

\begin{table*}[!ht]
\caption{The three typical graph scan statistics that are tested in our experiments (The vectors ${\bf x}$, ${\bf c}$, and ${\bf b}$ are defined in Section~\ref{sect:applications})}
\vspace{-2mm}
\centering
\begin{tabular}{ | p{2.8cm}|p{4cm}| p{9cm}|}
\hline
 \textbf{Score Functions} & \textbf{Definition} & \textbf{Applications} \\ 
\hline
 Kulldorff's Scan Statistic ~\cite{neill2009empirical}& ${\bf c}^\mathsf{T} {\bf x} \log \frac{{\bf c}^\mathsf{T} {\bf x}}{{\bf b}^\mathsf{T} {\bf x}}  -{\bf 1}^\mathsf{T} {\bf c} \log \frac{{\bf 1}^\mathsf{T} {\bf c}}{{\bf 1}^\mathsf{T} {\bf b}} + ({\bf 1}^\mathsf{T} {\bf c} - {\bf c}^\mathsf{T} {\bf x}) \log \frac{{\bf 1}^\mathsf{T} {\bf c} - {\bf c}^\mathsf{T} {\bf x}}{{\bf 1}^\mathsf{T} {\bf b} - {\bf b}^\mathsf{T} {\bf x}}$  & The statistic is used for anomalous pattern detection in graphs with count features, such as detection of traffic bottlenecks in sensor networks~\cite{anbarouglu2015non, modarres2007hotspot}, detection of anomalous regions in digitals and images~\cite{coulston2003geographic}, detection of attacks in computer networks~\cite{neil2013scan}, disease outbreak detection~\cite{Speakman-14}, and various others.\\ 
 \hline
 Expectation-based Poisson Statistic (EBP) ~\cite{neill2012fast}& ${\bf c}^\mathsf{T} {\bf x} \log \frac{{\bf c}^\mathsf{T} {\bf x}}{{\bf b}^\mathsf{T} {\bf x}} + {\bf b}^\mathsf{T} {\bf x} - {\bf c}^\mathsf{T} {\bf x}$ & This statistic is used for the same applications as above~\cite{gorr2015early, anbarouglu2015non, neill2009empirical, Speakman-14, speakman2013dynamic}, but has different assumptions on data distribution~\cite{neill2009empirical}.\\ 
 \hline 
 Elevated Mean Scan Statsitic (EMS) ~\cite{qian2014connected}& $ {\bf c}^\mathsf{T} {\bf x} / {\bf 1}^\mathsf{T} {\bf x}$ & This statistic is used for anomalous pattern detection in graphs with numerical features, such as event detection in social networks, network surveillance, disease outbreak
detection, biomedical imaging~\cite{qian2014connected, sharpnack2013near} \\ 
 \hline
\end{tabular}
\label{table:scoreFunctions}
\vspace{-3mm}
\end{table*}

\section{Applications on Graph Scan Statistics}
\label{sect:applications}
In this section, we specialize \textsc{Graph}-\textsc{IHT} and \textsc{Graph}-\textsc{GHTP} to optimize a number of well-known graph scan statistics for the task of connected subgraph detection, including elevated mean scan (EMS) statistic~\cite{qian2014connected}, Kulldorff's scan statistic~\cite{neill2009empirical}, and expectation-based Poisson (EBP) scan statistic~\cite{neill2012fast}. Each graph scan statistic is defined as the generalized likelihood ratio test (GLRT) statistic of a \textit{specific} hypothesis testing about the distributions of features of normal and abnormal nodes. The EMS statistic corresponds to the following GLRT test: Given a graph $\mathbb{G} = (\mathbb{V}, \mathbb{E})$, where $\mathbb{V} = [n]$ and $\mathbb{E} \subseteq \mathbb{V}\times \mathbb{V}$, each node $i$ is associated with a random variable $x_i$:
\begin{eqnarray}
x_i = \mu \cdot 1(i\in S) + \epsilon_i,\  i \in \mathbb{V},
\end{eqnarray} 
where $|\mu|$ represents the signal strength and $\epsilon_i \in \mathcal{N}(0, 1)$.  $S$ is some unknown anomalous cluster that forms as a connected subgraph. The task is to decide between the null hypothesis ($H_0$): $c_i \in \mathcal{N}(0, 1), \forall i \in \mathbb{V}$ and the alternative ($H_1(S)$): $c_i \in \mathcal{N}(\mu, 1), \forall i \in S$ and $c_i \in \mathcal{N}(0, 1), \forall i \notin S$. The EMS statistic is defined as the GLRT function under this hypothesis testing: 
\begin{eqnarray}
F(S) = \frac{\text{Prob}(\text{Data} | H_1(S))}{\text{Prob}(\text{Data} | H_0)} = \frac{1}{\sqrt{|S|}} \sum_{i \in S} c_i. 
\end{eqnarray}
The problem of connected subgraph detection based on the EMS statistic is then formulated as 
\begin{eqnarray}
\min_{S \subseteq \mathbb{V}} -\frac{1}{|S|} (\sum_{i \in S} c_i)^2\ \ s.t.\ \ S \in \mathbb{M}(\mathbb{G}, k), \label{EMS-detection}
\end{eqnarray}
where the \textbf{square} of the EMS scan statistic is considered to make the function smooth, and this transformation does not infect the optimum solution.  Let the $\{0,1\}$-vectors form of $S$ be ${\bf x}\in \{0, 1\}^n$, such that $\text{supp}({\bf x}) = S$. Problem~(\ref{EMS-detection}) can be reformulated as 
\begin{eqnarray}
\min_{x \in \{0, 1\}^n} - ({\bf c}^\mathsf{T} {\bf x})^2/ ({\bf 1}^\mathsf{T} {\bf x}) \ \ s.t.\ \text{supp}({\bf x}) \in \mathbb{M}(\mathbb{G}, k),
\end{eqnarray}
where ${\bf c} = [c_1, \cdots, c_n]^\mathsf{T}$. To apply our proposed algorithms, we relax the input domain of ${\bf x}$ and maximize the strongly convex function~\cite{bach2011learning}: 
\begin{eqnarray}
\min_{x \in \mathbb{R}^n} - ({\bf c}^\mathsf{T} {\bf x})^2/ ({\bf 1}^\mathsf{T} {\bf x}) + \frac{1}{2}{\bf x}^\mathsf{T} {\bf x}\ \ s.t.\ \text{supp}({\bf x}) \in \mathbb{M}(\mathbb{G}, k).\label{prob:stronglyconvex}
\end{eqnarray}
The connected subset of nodes can be found as the subset of indexes of positive entries in $\hat{\bf x}$, where $\hat{\bf x}$ refers to the solution of the Problem~(\ref{prob:stronglyconvex}). Assume that ${\bf c}$ is normalized and $c_i \le 1,$ $\forall i$. Let $\hat{c} = \max \{c_1, \cdots, c_n\}$. The Hessian matrix of the above objective function satisfies the following conditions 
\begin{eqnarray}
 (1 - \hat{c}^2) \cdot  \textbf{I}\preceq \textbf{I} - ({\bf c} - \frac{{\bf c}^\mathsf{T} {\bf x}}{{\bf 1}^\mathsf{T} {\bf x}} {\bf 1} ) ({\bf c} - \frac{{\bf c}^\mathsf{T} {\bf x}}{{\bf 1}^\mathsf{T} {\bf x}} {\bf 1})^\mathsf{T} \preceq 1\cdot \textbf{I}. 
\end{eqnarray}
According to Lemma 1 (b) in \cite{yuan2013gradient}),  the objective function $f({\bf x})$ satisfies condition $(\xi, \delta, \mathbb{M}(\mathbb{G}, 8k))$-WRSC that 
\[
\delta = \sqrt{1 - 2 \xi (1 - \hat{c}^2) + \xi^2},
\]
for any $\xi$ such that $\xi < 2 (1 - \hat{c}^2)$. The geometric convergence of \textsc{Graph-GTHP} as shown in Theorem~\ref{theorem-convergence:Graph-GHTP} is guaranteed.

Different from the EMS statistic that is defined for numerical features based on Gaussian distribution, the Kulldorff's scan statistic and Expectation Based Poisson statistic (EBP) are defined for count features based on Poisson distribution. In particular, each node $i$ is associated with a feature $c_i$, the count of events (e.g., crimes, flu infections) observed at the current time, and a feature $b_i$, the expected count (or `baseline') of events by using historical data. Let ${\bf c} = [c_1, \cdots, c_n]^\mathsf{T}$ and ${\bf b} = [b_1, \cdots, b_n]^\mathsf{T}$. The Kulldorff's scan statistic and EBP scan statistics are described Table~\ref{table:scoreFunctions}. We note that these two scan statistics do not satisfy the \textsc{WRSC} condition, but as demonstrated in our experiments, our proposed algorithms perform empirically well for all the three scan statistics, and in particular, our proposed \textsc{Graph-GHTP} converged in less than 10 iterations in all the settings.  

\begin{table*}[ht]
\centering \caption{Summary of dataset settings in the experiments. (For the network of each dataset, we only use its maximal connected component if it is not fully connected).} 
\vspace{-2mm}
\label{table:datasets}
\begin{threeparttable}
\resizebox{1\textwidth}{!}{
\begin{tabular}{|p{1.5cm}|p{2.6cm}|p{4.5cm}|p{2.4cm}|r|r|r|r|}
\hline
\textbf{Dataset}  & \textbf{Application} & \textbf{Training \& Testing Time Periods} &\textbf{Observed value at node }$v$: $c_v$ & \textbf{Baseline value at node }$v$: $b_v$ & \textbf{\# of Nodes} & \textbf{\# of Edges} & \textbf{\# of snapshots}  \\
\hline
BWSN & Detection of  & Training: Hours 3 to 5 with $0\%$ noise & Sensor value (0 or 1) & Average sensor value for EBP& 12,527 & 14,323 & hourly: $3\times 6$  \\
 & contaminated nodes & Testing: Hours 3 to 5 with $2\%,4\%, \cdots, 10\%$ noises &  & Constant `1' for Kulldorff & & &\\
\hline
CitHepPh & Detection of emerging & Testing: 1999 to 2002 &  Count of citations & Average count of citations for EBP & 11,895 & 75,873 & yearly: $1\times 11$ \\
&  research areas&  &  & Maximum count of citations for Kulldorff & & & \\
\hline
Traffic & Detection of most &Testing: Mar.2014, 5AM to 10PM &  $-\log(p^t(v)/\mu)\ $\footnotemark & None  & 1,723 & 5,301 &per 15 min: $68\times 304$  \\
& congested subgraphs  &  &  & (EBP and Kulldorff are not applicable) & & &  \\
\hline
ChicagoCrime & Detection of crime & Testing: Year of 2015 & Count of burglaries & Average count of burglaries for EBP & 46,357 & 168,020 & yearly: $1\times 15$ \\
& hot spots& & & Maximum count of burglaries for Kulldorff & & &  \\
\hline
\end{tabular}}
\begin{tablenotes}
\item[1] $p^t(v)$ refers to the statistical p-value of node $v$ at time $t$ that is calculated via empirical calibration based on historical speed values of $v$ from 2013 June. \\ 1st  to 2014 Feb. 29; and $\mu$ is a significance level threshold and is set to $0.15$. The larger this value $-\log(p^t(v)/\mu)$, the more congested in the region near $v$. 
\end{tablenotes}
\end{threeparttable}
\vspace{-3mm}
\end{table*}

\section{experiments}
\label{sect:experiments}
This section evaluates the performance of our proposed methods using four public benchmark data sets for  connected subgraph detection. The \textbf{experimental code and data sets} are available from the Link \cite{fullversion} for reproducibility. 
\subsection{Experiment Design}
\label{sect:experiment-design}
{\bf Datasets:} 1) \textbf{BWSN Dataset.} A real-world water network is offered in the Battle of the Water Sensor Networks (BWSN)~\cite{ostfeld2008battle}. That has 12,527 nodes and 14,323 edges. In order to simulate a contaminant sub-area, 4 nodes with chemical contaminant plumes, which were distributed in this sub-area, were generated. We use the water network simulator EPANET~\cite{rossman2000epanet} that was employed in BWSN for a period of 3 hours to simulate the spreads for contaminant plumes on this graph. If a node is polluted by the chemical, then its sensor reports 1, otherwise, 0, in each hour. To test the tolerance of noise of our methods, $K\in \{2,4,6,8,10\}$ percent vertices were selected randomly, and their sensor reports were set to 0 if their original reports were 1 and vice versa. Each hour has a graph snapshot. The snapshots corresponding to the 3 hours that have 0\% noise are considered for training, and the snapshots that have $2\%, \cdots, 10\%$ noise reports for testing. The goal is to detect a connected subgraph that is corresponding to the contaminant sub-area. 2) \textbf{CitHepPh Dataset.} We downloaded the high energy physics phenomenology citation data (CitHepPh) from Stanford Network Analysis Project  (SNAP) \cite{leskovec2005graphs}. This citation graph contains 11,897 papers corresponding to graph vertices and 75,873 edges. An undirected edge between two vertices (papers) exists , if one paper is cited by another. The period of these papers published is from January 1992 to April 2002. Each vertex has two attributes for each specific year ($t=1992,\cdots ,t = 2002$). We denote the number of citations of each specific year as the first attribute and the average citations of all papers in that year as the second attribute. The goal is to detect a connected subgraph where the number of citations of vertices (papers) in this subgraph are abnormally high compared with the citations of vertices that are not in this subgraph. This connected subgraph is considered as a potential emerging research area. Since the training data is required for some baseline methods, the data before 1999 is considered as the training data, and the rest from 1999 to 2002 as the testing data.
    3) \textbf{Traffic Dataset}. Road traffic speed
data from June 1, 2013 to Mar. 31, 2014 in the arterial road network of the Washington D.C. region is collected from the INRIX database (http://inrix.com/publicsector.asp), with 1,723  nodes and 5,301 edges. The database
provides traffic speed for each link at a 15-minute rate. For each 15-minute interval, each method identities a connected subgraph as the most congested region. 4) \textbf{ChicagoCrime Dataset.} We collected crime data from City of Chicago [https://data.cityofchicago.org/Public-Safety/Crimes-2001-to-present/ijzp-q8t2] from Jan. 2001 to Dec. 2015. There are 46,357 nodes (census blocks) and 168,020 edges (Two census blocks are connected with each other if they are neighbours). Specifically, we collected all records of burglaries from 2001 to 2015. The data covers burglaries in the period from Jan. 2001 to Dec. 2015. Each vertex has an attribute denoting the number of burglaries in sepcific year and average number of burglaries over 10 years. We aim to detect connected census areas which has anomaly high burglaries accidents. The data before 2010 is considered as training data, and the data from 2011 to 2015 is considered as testing data.

{\bf Graph Scan Statistics: } As shown in Table~\ref{table:scoreFunctions}, three graph scan statistics were considered as the scoring functions of connected subgraphs, including Kulldorff's scan statistic~\cite{neill2009empirical}, expectation-based Poisson (EBP) scan statistic~\cite{neill2012fast}, and elevated mean scan (EMS) statistic~\cite{qian2014connected}. The first two statistic functions require that each vertex $v$ has a count $c_v$ representing the count of events observed at that vertex, and an expected count (`baseline`) $b_v$. For EMS statistic, only $c_v$ is used. We need to normalize $c_v$ for EMS as it is defined based on the assumptions of standard normal distribution for normal values and shifted-mean normal distribution for abnormal values. Table~\ref{table:datasets} provides details about the calculations of $c_v$ and $b_v$ for each data set.

\begin{figure*}[!htb]
\begin{minipage}{0.35\textwidth}
\centering
  \includegraphics[width=0.80\textwidth, height=0.91\textwidth,natwidth=610,natheight=642]{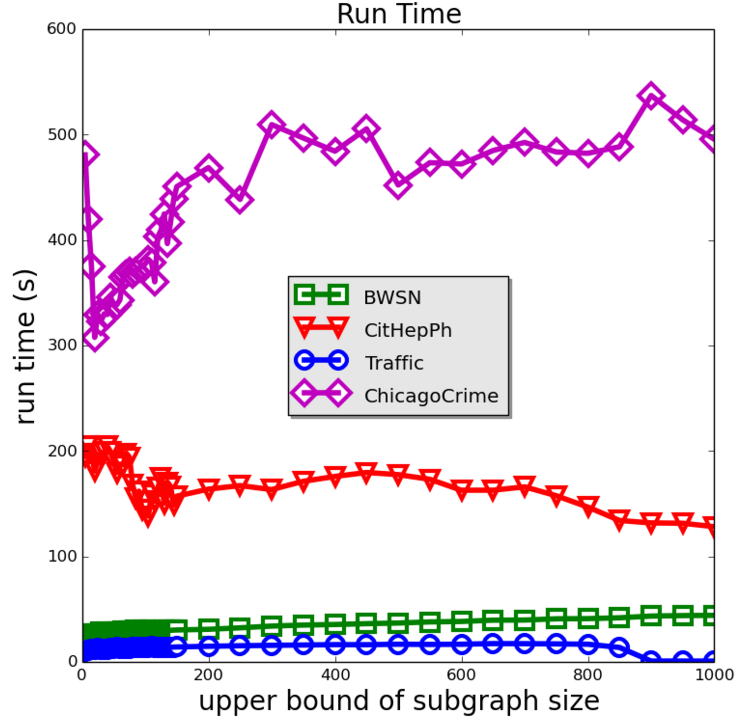}
  \caption{Scability of \textsc{Graph-GHTP} with respect to  $k$ (the upper bound of subgraph size).}
  \label{fig:comparsion-of-iterations1}
\end{minipage}\hfill
\begin{minipage}{0.63\textwidth}
\centering
   \includegraphics[width=0.95\textwidth, height=0.51\textwidth,natwidth=610,natheight=642]{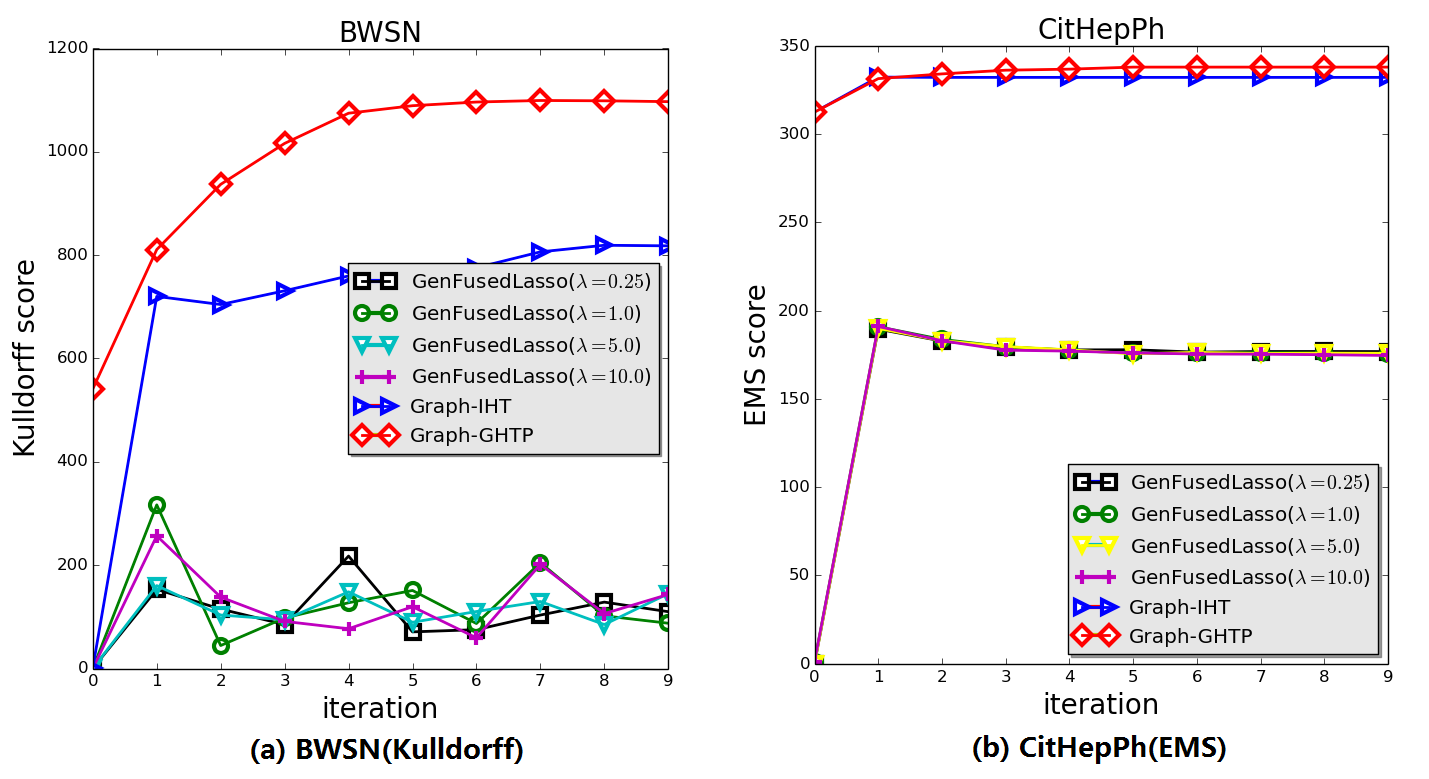}
  \caption{Evolving curves of graph scan statistic scores between our methods( \textsc{Graph-IHT} and \textsc{Graph-GHTP}) and \texttt{GenFusedLasso} in different iterations.}
  \label{fig:comparsion-of-iterations2}
\end{minipage}
\vspace{-1mm}
\end{figure*}

\begin{figure*}[t]
  \centering
  \subfigure[BWSN(precision)]{
    \includegraphics[width=0.31\textwidth, height=0.26\textwidth,natwidth=610,natheight=642]{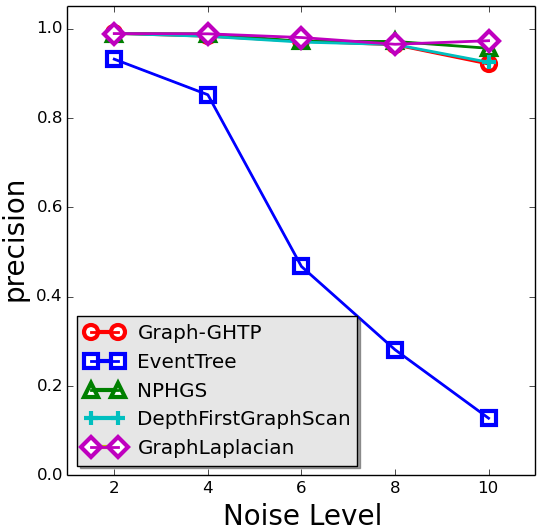}}
  \subfigure[BWSN(recall)]{
    \includegraphics[width=0.31\textwidth, height=0.26\textwidth,natwidth=610,natheight=642]{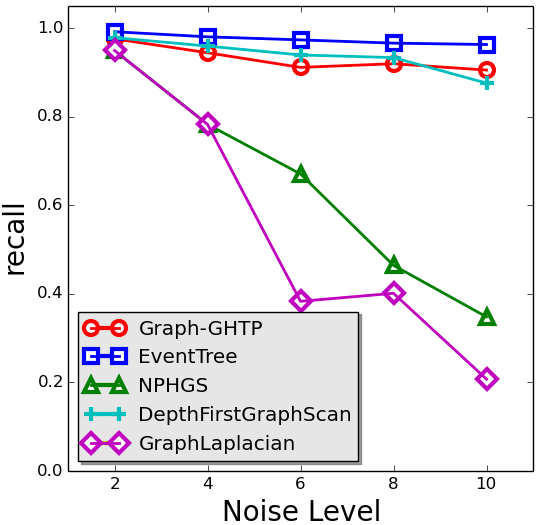}}
    \subfigure[BWSN(fmeasure)]{
    \includegraphics[width=0.31\textwidth, height=0.26\textwidth,natwidth=610,natheight=642]{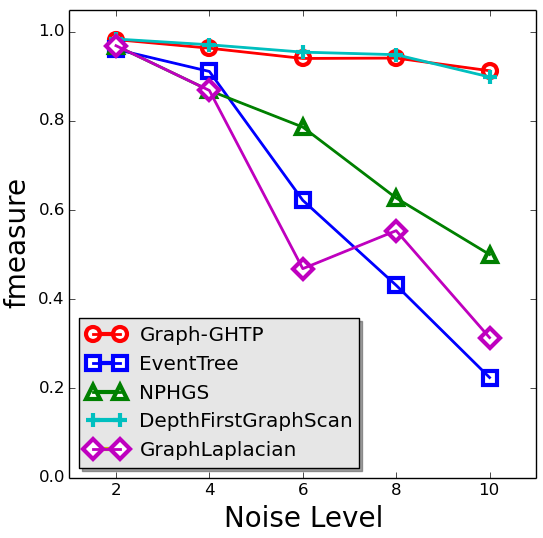}}
  \caption{Precision, Recall, and F-measure curves for water pollution detection in BWSN with respect to different noise ratios.}
  \label{fig:comparsion-of-pre-rec}
  
\end{figure*}

\begin{table*}[t]
\caption{Comparison on scores of the three graph scan statistics based on connected subgraphs returned by comparison methods. EMS and EBP are Elevated Mean Scan Statistic and Expectation-Based Poisson Statistic, respectively. }
\vspace{-2mm}
\small
\footnotesize
\centering

\begin{tabular}{|c | c | c | c | c | c |c | c | c |}
\hline
 &\multicolumn{4}{c|}{BWSN}&\multicolumn{4}{c|}{CitHepPh}\\
\cline{2-9}
  & Kulldorff & EMS & EBP & Run Time & Kulldorff & EMS & EBP & Run Time\\
\hline
\hline 
\textsc{Graph-GHTP}   &\textbf{1097.15} & \textbf{21.56} & \textbf{79.71} & \textbf{165.86} & 16296.40 & \textbf{337.90} & \textbf{9342.94} & 155.74\\

\texttt{GraphLaplacian} & 474.96 & 14.89 & 49.91 & 55315.94 & 2585.44 & 202.38 & 2305.05 & 22424.24\\

\texttt{EventTree} & 834.59 & 20.25 & 32.13 & 441.74 & \textbf{16738.43} & 335.34 & 9061.56 & \textbf{124.28}\\

\texttt{DepthFirstGraphScan} & 735.85 & 20.41 & 79.30 & 5929.00 & 9531.19 & 260.06 & 5561.66 & 12183.88\\

\texttt{NPHGS} & 541.13 & 16.90 & 58.59 & 256.91 & 11965.14 & 326.23 & 9098.22 & 175.08\\

\hline
\end{tabular}
\label{table:comparison1}
\vspace{-3mm}
\end{table*}

\begin{table*}[t]
\caption{Comparison on scores of the three graph scan statistics based on connected subgraphs returned by comparison methods. EMS and EBP are Elevated Mean Scan Statistic and Expectation-Based Poisson Statistic, respectively. \texttt{GraphLaplacian} failed to run on ChicagoCrime due to out-of-memory error.}
\vspace{-2mm}
\small
\footnotesize
\centering

\begin{tabular}{|c | c | c | c |c | c | c |}
\hline
 &\multicolumn{2}{c|}{Traffic}&\multicolumn{4}{c|}{ChicagoCrime}\\
\cline{2-7}
  & EMS & Run Time & Kulldorff & EMS & EBP & Run Time\\
\hline
\hline
\textsc{Graph-GHTP}   & \textbf{20.45} & 22.25 & \textbf{6386.08} & \textbf{5.45} & \textbf{5172.54} & 3177.60\\

\texttt{GraphLaplacian} & 5.40 & 291.75 & - & - & - & -\\

\texttt{EventTree} & 12.40 & 5.02 & 4388.42 & 4.91 & 3965.96 & \textbf{226.50}\\

\texttt{DepthFirstGraphScan} & 8.13 & 47.73 & 1123.49 & 2.56 & 1094.21 & 12133.50\\

\texttt{NPHGS} & 6.28 & \textbf{0.22} & 966.70 & 2.43 & 948.23 & 701.40\\
\hline
\end{tabular}
\label{table:comparison2}
\vspace{-3mm}
\end{table*}

{\bf Comparison Methods:} We compared our proposed methods with four state-of-the-art baseline methods that are designed specifically for connected subgraph detection, namely, \texttt{GraphLaplacian}  ~\cite{sharpnack2012changepoint}, \texttt{EventTree} ~\cite{rozenshtein2014event}, \texttt{DepthFirstGraphScan} ~\cite{Speakman-14} and \texttt{NPHGS} ~\cite{DBLP:conf/kdd/ChenN14}.  \texttt{DepthFirstGraphScan} is an exact search algorithm based on depth-first search and takes weeks to run on graphs that have more than 1000 nodes. We imposed a maximum limit on the depth of the search to 10 to reduce its time complexity. 

\vspace{-1mm}
The basic ideas of these baseline methods are summarized as follows: \texttt{NPHGS} starts from random seeds (nodes) as initial candidate clusters and gradually expends each candidate cluster by including its neighboring nodes that could help improve its BJ statistic score until no new nodes can be added. The candidate cluster with the largest BJ statistic score is returned. \texttt{DepthFirstGraphScan} adopts a similar strategy to \texttt{NPHGS} but expands the initial clusters based on depth-first search. \texttt{GraphLaplacian} uses a graph Laplacian penalty function to replace the connectivity constraint and converts the problem to a convex optimization problem. \texttt{EventTree} reformulates the connected subgraph detection problem as a prize-collecting steiner tree (PCST) problem~\cite{johnson2000prize} and apply the Goemans-Williamson (G-W) algorithm for PCST~\cite{johnson2000prize} to detect anomalous subgraphs. We also implemented the generalized fused lasso model (\texttt{GenFusedLasso}) for graph scan statistics  using the framework of alternating direction method of multipliers (ADMM). \texttt{GenFusedLasso} method solves the following minimization problem
\begin{eqnarray}
\min_{{\bf x} \in \mathbb{R}^n} -f({\bf x}) + \lambda \sum\nolimits_{(i, j) \in \mathbb{E}} |x_i - x_j|,
\end{eqnarray}
where $f({\bf x})$ is a predefined graph scan statistic and the trade-off parameter $\lambda$ controls the degree of smoothness of neighboring entries in ${\bf x}$. We applied the heuristic rounding step proposed in~\cite{qian2014connected} to ${\bf x}$ to generate connected subgraphs. 

{\bf Parameter Tunning: } We
strictly followed strategies recommended by authors in their original
papers to tune the related model parameters.
Specifically, for \texttt{EventTree}, we tested the set of $\lambda$ values: $\{0.02, 0.04, \cdots, 2.0, 3.0, \cdots, 20\}$. For \texttt{Graph-Laplacian}, we tested the set of $\lambda$ values: $\{0.001, 0.003, 0.01, 0.03, 0.1, 0.3, 1\}$ and returned the best result. For \texttt{GenFusedLasso}, we tested the set of $\lambda$ values: $\{0.02, 0.04,$ $\cdots, 2.0, 3.0, \cdots, 20\}$. For \texttt{NPHGS}, we set the suggested parameters by the authors: $\alpha_{max} = 0.15$ and $K=5$. Our proposed methods \textsc{Graph-IHT} and \textsc{Graph-GHTP} have a single parameter $k$, an upper bound of the subgraph size. We tested the set of $k$ values: $\{50, 100, \cdots, 1000\}$. As the BWSN dataset has the ground truth about the contaminated nodes, we identified the best parameter for each method that has the largest F-measure in the training data. For the other data sets, as we do not have ground truth about the true subgraphs, for each specific scan statistic, we identified the best parameter for each method that was able to identify the connected subgraph with the largest statistic score. 

{\bf Performance Metrics: } \textbf{1) Optimization Power.} The overall scores of the three graph scan statistic functions of the connected subgraphs returned by the comparison methods are compared and analyzed. The objective is to identify methods that could find the connected subgraphs with the largest graph scan statistic scores. \textbf{2) Precision, Recall, and F-Measure.} For the BWSN dataset, as the true anomalous subgraphs are known,we use F-measure that combines precision and recall to evaluate the quality of detected subgraphs by different methods. \textbf{3) Run Time.} The running times of different methods are compared. 

\subsection{Evolving Curves of Graph Scan Statistics}

Figure~\ref{fig:comparsion-of-iterations2} compares our methods \texttt{Graph-IHT} and \textsc{Graph-GHTP} with \texttt{GenFusedLasso} on the scores of two graph scan statistics (Kulldorff's scan statistic and elevated mean scan statistic (EMS)) based on the best connected subgraphs identified by both methods in different iterations. Note that, a heuristic rounding process as proposed in~\cite{qian2014connected} is applied to continuous vector ${\bf x}^i$ estimated by \texttt{GenFusedLasso} at each iteration $i$ in order to identify the best connected subgraph at the current iteration. As the setting of the parameter $\lambda$ will influence the quality of the detected connected subgraph, the results under different $\lambda$ values are also shown in Figure~2. The results indicate that our method \textsc{Graph-GHTP} converges in less than 10 iterations, and \textsc{Graph-IHT} converges in more steps. The qualities (scan statistic scores) of the connected subgraphs identified at different iterations by our two methods are consistently higher than those returned by \texttt{GenFusedLasso}. 

\subsection{Optimization Power}
The comparisons between our method and the other baseline methods are shown in Table~\ref{table:comparison1} and Table~\ref{table:comparison2}.
The scores of the three graph scan statistics based on the connected subgraphs returned by these methods are reported in these two tables. The results in indicate that our method outperformed all the baseline methods on the three graph scan statistics, except that \texttt{EventTree} achieved the highest Kulldorff score (16738.43) on the CitHepPh dataset, but is only $2.71\%$ larger than the returned score of our method \textsc{Graph-GHTP}. We note \texttt{EventTree} is a heuristic algorithm and does not provide theoretical guarantee on the quality of the connected subgraph returned, as measured by the scan statistic scores. 

\subsection{Water Pollution Detection}
Figure~\ref{fig:comparsion-of-pre-rec} shows the precision, recall, and F-measure of all the comparison methods on the detection of polluted nodes in the water distribution network in BWSN with respect to different noise ratios. The results indicate that our proposed method \textsc{Graph-GHTP} and \texttt{DepthFirstGraphScan} were the best methods on all the three measures for most of the settings. However, \texttt{DepthFirstGraphScan} spent $5929$ seconds to finish, and \textsc{Graph-GHTP} spent only $166$ seconds, 35.8 times faster than \texttt{DepthFirstGraphScan}. 
\texttt{EventTree} achieved high recalls but low precisions consistently in different settings. In contrast, \texttt{GraphLaplacian} and \texttt{NPHGS} achieved high precisions but low recalls in most settings. 

\subsection{Scalability Analysis}
Table~\ref{table:comparison1} and Table~\ref{table:comparison2} also show the comparison between our proposed method \textsc{Graph-GHTP} and other baseline methods on the running time. The results indicate that our proposed method \textsc{Graph-GHTP} ran faster than all the baseline methods in most of the settings, except for \texttt{EventTree}. \texttt{EventTree} was the fastest method but was unable to detect subgraphs with high qualities. As our method has a parameter on the upper bound ($k$) of the subgraph returned, we also conducted the scalability of our method with respect to different values of $k$ as shown in Figure \ref{fig:comparsion-of-iterations1}. The results indicate that the running time of our algorithm is insensitive to the setting of $k$, which is consistent with the time complexity analysis of \textsc{Graph-GHTP} as discussed in Theorem~\ref{Theorem:runningTime}.

\section{Related Work}
\label{sect:relatedWork}
\noindent \textbf{A. Structured sparse optimization.} The methods in this category have been briefly reviewed in the introduction section. The most relevant work is by Hegde et al.~\cite{hegde2015nearly}. The authors present \textsc{Graph}-\textsc{Cosamp}, a variant of \textsc{Cosamp}~\cite{needell2009cosamp}, for compressive sensing and linear regression problems based on head and tail approximations of $\mathbb{M}(\mathbb{G}, k)$. 

\vspace{1mm}
\noindent \textbf{B. Connected subgraph detection.} Existing methods in this category fall into three major categories:1)  \underline{Exact algorithms.} The most recent method is a brunch-and-bounding algorithm \texttt{DepthFirstGraphScan}~\cite{Speakman-14} that runs in exponential time in the worst case; 2)\underline{ Heuristic algorithms.} The most recent methods in this category include \texttt{EventTree}~\cite{rozenshtein2014event}, \texttt{NPHGS} ~\cite{DBLP:conf/kdd/ChenN14}, \texttt{AdditiveScan} ~\cite{speakman2013dynamic}, \texttt{GraphLaplacian} ~\cite{sharpnack2012changepoint}, and \texttt{EdgeLasso} ~\cite{sharpnack2012sparsistency}; 3) \underline{Approximation algorithms that provide performance bounds.} The most recent method is presented by Qian et al.. The authors reformulate the connectivity constraint as linear matrix inequalities (LMI) and present a semi-definite programming algorithm based on convex relaxation of the LMI [18, 19] with a performance bound. However, this method is not scalable to large graphs ($\ge$ 1000 nodes). Most of the above methods are considered as baseline methods in our experiments and are briefly summarized in Section~\ref{sect:experiment-design}.

\section{conclusion}
\label{sect:conclusion}
This paper presents, \textsc{Graph-IHT} and \textsc{Graph-GHTP}, two efficient algorithms to optimize a general nonlinear optimization problem subject to connectivity constraint on the support of variables. Extensive experiments demonstrate the effectiveness and efficency of our algorithms. For the future work, we plan to explore graph-structured constraints other than connectivity constraint and extend our proposed methods such that good theoretical properties of the cost functions that do not satisfy the WRSC condition can also be analyzed.



%
\bibliographystyle{abbrv}
\begin{small}
\bibliography{icdm2016.bib}  
\end{small}

\end{document}